\documentclass{article} % For LaTeX2e
\clearpage
\usepackage{iclr2026_conference,times}
\usepackage{microtype}
\usepackage{graphicx}
\usepackage{subcaption}
\usepackage{booktabs} % for professional tables
\usepackage{multirow}
\usepackage{makecell}
\usepackage{adjustbox}
\usepackage{threeparttable}
\usepackage{caption}
\usepackage{wrapfig}
\usepackage{enumitem}
\usepackage{amssymb}
\usepackage{amsthm}
\newtheorem{theorem}{Theorem}
\usepackage{amsmath,amsfonts,bm}

\def\eqref#1{equation~\ref{#1}}
\def\1{\bm{1}}

\DeclareMathAlphabet{\mathsfit}{\encodingdefault}{\sfdefault}{m}{sl}
\SetMathAlphabet{\mathsfit}{bold}{\encodingdefault}{\sfdefault}{bx}{n}

\usepackage{hyperref}
\usepackage{url}

\usepackage[table]{xcolor}
\setlist[itemize]{itemsep=3pt, topsep=6pt, leftmargin=1.5em}
\usepackage{hhline}
\definecolor{TableHeader}{RGB}{210, 225, 235}
\usepackage[breakable]{tcolorbox}
\usepackage{fvextra} % for Verbatim with automatic line breaking
\usepackage{pgfplots}
\pgfplotsset{compat=1.18}
\makeatletter
\providecommand{\FloatBarrier}{%
  \ifx\@deferlist\@empty\else\clearpage\fi}
\makeatother

\usepackage{xspace}

\newcommand{\ourmethod}{EvoKernel\xspace}

\usepackage{amsmath}
\usepackage{mathtools}
\newtheorem{lemma}[theorem]{Lemma}
\newtheorem{corollary}[theorem]{Corollary}
\theoremstyle{definition}

\theoremstyle{remark}
\newtheorem{remark}[theorem]{Remark}

\usepackage[capitalize, noabbrev]{cleveref}

\makeatletter
\renewenvironment{abstract}{%
  \centerline{\large\textsc{Abstract}}
  \vspace{1ex}%
  \begin{quote}%
}{%
  \end{quote}%
  \par\bigskip %
}
\makeatother

\iclrfinalcopy

\title{Towards Cold-Start Drafting and Continual Refining: A Value-Driven Memory Approach with Application to NPU Kernel Synthesis}

\author{%
  \textbf{Yujie Zheng\textsuperscript{*,\,1}},\hspace{0.6em} \textbf{Zhuo Li\textsuperscript{*,\,1}},\hspace{0.6em} \textbf{Shengtao Zhang\textsuperscript{1}},\hspace{0.6em}
  \textbf{Hanjing Wang\textsuperscript{2}},\hspace{0.6em}
  \textbf{Junjie Sheng\textsuperscript{3}},\hspace{0.6em}
  \textbf{Jiaqian Wang\textsuperscript{1}}, \\
  \textbf{Junchi Yan\textsuperscript{1}},\hspace{0.6em}
  \textbf{Weinan Zhang\textsuperscript{1}},\hspace{0.6em} \textbf{Ying Wen\textsuperscript{1}},\hspace{0.6em} \textbf{Bo Tang\textsuperscript{4}},\hspace{0.6em} \textbf{Muning Wen\textsuperscript{\textdagger,\,1}}\\[1em]
  \normalfont
  \textsuperscript{1}Shanghai Jiao Tong University \quad \textsuperscript{2}Shanghai Artificial Intelligence Laboratory \\
  \textsuperscript{3}Independent Researcher \quad \textsuperscript{4}MemTensor (Shanghai) Technology Co., Ltd
}

\begin{document}

\maketitle
\renewcommand{\thefootnote}{\fnsymbol{footnote}}
\footnotetext[1]{Equal contribution.}
\footnotetext[2]{Corresponding author: Muning Wen (muningwen@sjtu.edu.cn)}
\thispagestyle{plain}
\pagestyle{plain}

\vskip 0.1in

\begin{abstract}
Deploying Large Language Models to data-scarce programming domains poses significant challenges, particularly for kernel synthesis on emerging Domain-Specific Architectures where a ``Data Wall'' limits available training data.
While models excel on data-rich platforms like CUDA, they suffer catastrophic performance drops on data-scarce ecosystems such as NPU programming.
To overcome this cold-start barrier without expensive fine-tuning, we introduce \textbf{\ourmethod}, a self-evolving agentic framework that automates the lifecycle of kernel synthesis from initial drafting to continual refining.
\ourmethod addresses this by formulating the synthesis process as a memory-based reinforcement learning task. Through a novel value-driven retrieval mechanism, it learns stage-specific Q-values that prioritize experiences based on their contribution to the current objective—whether bootstrapping a feasible draft or iteratively refining latency.
Furthermore, by enabling cross-task memory sharing, the agent generalizes insights from simple to complex operators.
By building an NPU variant of KernelBench and evaluating on it, \ourmethod improves frontier models' correctness from 11.0\% to 83.0\% and achieves a median speedup of 3.60$\times$ over initial drafts through iterative refinement. This demonstrates that value-guided experience accumulation allows general-purpose models to master the kernel synthesis task on niche hardware ecosystems. Our official page is available at \url{https://evokernel.zhuo.li}.
\end{abstract}

\section{Introduction}
A practical limitation when deploying Large Language Models (LLMs) to niche domains is their inability to generalize beyond their pre-training distribution~\citep{minaee2024large, wang2024generalization}. When faced with \emph{cold-start} scenarios, domains where training data is sparse and expert demonstrations are unavailable, even frontier models struggle significantly~\citep{kostikova2025lllms, joel2410survey}. 
This challenge is particularly acute in domains where (i) correctness is binary and machine-verifiable, leaving little room for ``partially correct'' solutions~\citep{jain2024livecodebench, yan2024codescope}, (ii) expert knowledge is scarce and expensive to acquire, and (iii) the gap between in-distribution and out-of-distribution performance is stark.

Automated kernel synthesis for emerging hardware accelerators exemplifies this extreme scarcity~\citep{yu2026towards}. While the industry is aggressively diversifying toward Domain-Specific Architectures (DSAs) like NPUs, TPUs, and neuromorphic chips~\citep{silvano2025survey, liao2021ascend, jouppi2023tpu} to address escalating computational costs~\citep{kaplan2020scaling}, these nascent ecosystems face a severe ``Data Wall''. Unlike the mature NVIDIA landscape, where decades of CUDA repositories provide a massive pre-training corpus, emerging platforms are characterized by extreme data scarcity: public code is rare, documentation is esoteric, and compiler feedback is opaque~\citep{joel2410survey}. This barrier is compounded by the fact that highly optimized CUDA kernels~\citep{choquette2021nvidia, wu2023pytorch} are not portable to these architectures due to fundamental differences in memory hierarchy and instruction sets, leaving foundation models with virtually no expert demonstrations to bridge the cold-start gap.

\begin{table}[t]
    \centering
    \caption{Few-shot functional correctness (pass@4) of frontier LLMs on CUDA vs.\ Ascend C kernel generation. Results are from our experiments; the level definitions (L1, L2) and setup details are consistent with Section~\ref{sec:exp_setup}.}
    \label{tab:cuda_vs_ascendc}
    \setlength{\tabcolsep}{12pt}
    \begin{tabular}{lccc}
      \toprule
      \textbf{Model} & \textbf{Level} & \textbf{CUDA (\%)} & \textbf{Ascend C (\%)} \\
      \midrule
      \multirow{2}{*}{GPT-5.2}
        & L1 & 92.0 & 14.0 \\
        & L2 & 90.0 & 2.0 \\
      \midrule
      \multirow{2}{*}{DeepSeek-V3.2}
        & L1 & 50.0 & 8.0 \\
        & L2 & 9.0 & 0.0 \\
      \midrule
      \multirow{2}{*}{Qwen3-Coder-30B}
        & L1 & 46.0 & 7.0 \\
        & L2 & 10.0 & 0.0 \\
      \bottomrule
    \end{tabular}
\end{table}

As evidenced in Table~\ref{tab:cuda_vs_ascendc}, state-of-the-art LLMs that achieve high performance on CUDA~\citep{ouyang2025kernel} suffer a catastrophic collapse when transferred to a data-scarce Domain-Specific Language (DSL) like Ascend C, which is specifically designed for NPU kernel programming. In line with prior findings~\citep{wen2025multikernelbench}, even GPT-5.2, which attains 92\% on CUDA L1 tasks, drops to 14\% on Ascend C; on the more challenging L2 tasks, models fail entirely. This observation suggests that current models do not genuinely ``learn'' to program new hardware like NPUs, but instead rely on memorized patterns from pre-training distributions.

Standard paradigms to bridge this gap prove insufficient in such data-scarce domains. Supervised Fine-Tuning (SFT)~\citep{zhou2023lima, chung2024scaling} demands thousands of expert-labeled examples per domain~\citep{longpre2023flan}, which is prohibitively expensive when targeting rapidly evolving or niche environments like NPU programming. Parametric policy-based Reinforcement Learning~\citep{zhang2025settling, kakade2003sample} requires extensive online rollouts to update model weights, incurring high sample complexity~\citep{cao2024beyond, qi2025sample} and risking catastrophic forgetting of general capabilities. Traditional Retrieval-Augmented Generation (RAG)~\citep{lewis2020retrieval} falters when the database is sparse~\citep{contal2024ragsys, barnett2024seven}; even with relevant samples, similarity-based retrieval does not guarantee effectiveness~\citep{izacard2023atlas}. Consequently, the core challenge is a \textbf{cold-start} problem: \textit{How can an agent autonomously master a rigorous, data-scarce kernel synthesis task from scratch, without expert demonstrations or expensive fine-tuning?}

To address this, we introduce \textbf{\ourmethod}, a framework that formulates kernel synthesis as a reinforcement learning task over a self-evolving memory. By employing a novel value-driven retrieval mechanism, the agent learns stage-specific Q-values to quantify the utility of historical experiences, dynamically shifting focus from bootstrapping functional correctness (Drafting) to optimizing latency (Refining) without updating model weights. Empirically, EvoKernel bridges the cold-start gap on NPU benchmarks, boosting the correctness of frontier models \textit{from 11.0\% to 83.0\%} and achieving a \textit{3.60x median speedup} over the first feasible draft, thereby demonstrating that value-guided experience accumulation enables general-purpose models to master data-scarce hardware ecosystems.

Our contributions are summarized as follows:
\begin{itemize}
    \item \textbf{Unified Drafting-Refining Pipeline:} We propose a two-stage framework over a shared memory that transitions from feasibility-driven drafting to latency-driven refining to bootstrap and optimize NPU kernels. 
    \item \textbf{Evolving Value-Driven Retrieval:} We introduce a retrieval mechanism that learns stage-specific Q-values to quantify memory utility. A unified Monte-Carlo update adapts the policy from verifier feedback without updating model weights. 
    \item \textbf{Comprehensive Evaluation and Insights:} EvoKernel boosts performance on NPU benchmarks from 11.0\% to 83.0\%. We provide in-depth analysis of cross-task transfer, emergent curricula, and scaling to out-of-distribution workloads such as the Attention Set and recent MHC kernels, demonstrating how memory autonomously bridges the data-scarce gap. 
\end{itemize}

\section{Related Work}
\textbf{Self-Evolving and Adaptive Agents.}
While Large Language Models (LLMs) are typically static, recent research explores mechanisms for self-improvement.
Inference-time techniques, such as Self-Refine~\citep{madaan2023self} and Tree-of-Thoughts~\citep{yao2023tree}, utilize iterative critique loops to enhance reasoning within a single episode, though these improvements are transient, resetting once the context window closes~\citep{shinn2023reflexion}.
Closest to our work are evolutionary frameworks like AlphaEvolve~\citep{novikov2025alphaevolve} and EvolveR~\citep{wu2025evolver}, which accumulate experience across episodes. These methods typically assume sufficient initial competency or verifiable intermediate states, conditions absent in the rigid ``all-or-nothing'' compilation environment of data-scarce kernel synthesis, where our approach operates.

\textbf{Memory-Augmented Generation.}
To overcome context limitations, systems like MemGPT~\citep{packer2023memgpt} and MemOS~\citep{li2025memos, chhikara2025mem0} introduce operating-system-like memory hierarchies for long-horizon tasks. In agentic workflows, Voyager~\citep{wang2023voyager} and other generative agents~\citep{park2023generative, fang2025memp} demonstrate the power of retrieving procedural skills or behavioral reflections~\citep{madaan2022memprompt}. More recently, Memento~\citep{zhou2508memento} and MemRL~\citep{zhang2026memrl} have formalized retrieval as a reinforcement learning problem, learning what to retrieve. We adapt this value-based retrieval paradigm to kernel engineering, where surface-level semantic similarity often fails.

\textbf{Automated Kernel Synthesis.}
Kernel synthesis demands strict functional correctness and hardware-specific optimization. Benchmarks like KernelBench~\citep{ouyang2025kernel} and MultiKernelBench~\citep{wen2025multikernelbench} reveal that general-purpose LLMs degrade sharply on unfamiliar backends due to domain shifts~\citep{li2025tritonbench}. To mitigate this, recent agentic frameworks such as QiMeng-Kernel~\citep{zhu2025qimeng} and KernelBand~\citep{ran2025kernelband} utilize iterative execution feedback for refinement, as do multi-agent systems like STARK~\citep{dong2025stark} and AKG Kernel Agent~\citep{du2025akg}. Supervised approaches like Kevin~\citep{baronio2025kevin} and AutoTriton~\citep{li2025autotriton, woo2025tritonrl} fine-tune models on domain-specific corpora. These methods often assume access to high-quality training data, limiting their applicability in emerging ecosystems. EvoKernel addresses this cold-start setting by learning to retrieve from a self-evolving memory bank rather than relying on static corpora.

\section{EvoKernel: Value-Driven Memory Update for Kernel Evolution}

As shown in Figure~\ref{fig:framework}, we propose the EvoKernel, a framework that automates the lifecycle of hardware-specific kernel synthesis, from cold-start drafting to continual performance refinement. In this paper, we instantiate the framework primarily on Ascend C, while the same agent loop can be specialized to other backends through backend-specific prompts, verifier toolchains, and profiling signals. We formulate this process as a Memory-based Markov Decision Process (M-MDP)~\citep{zhou2508memento, zhang2026memrl}, where an agent learns to retrieve high-utility experiences to guide a LLM generator.

\label{sec:method}
\begin{figure*}[!htbp]
\centering
\includegraphics[width=\textwidth]{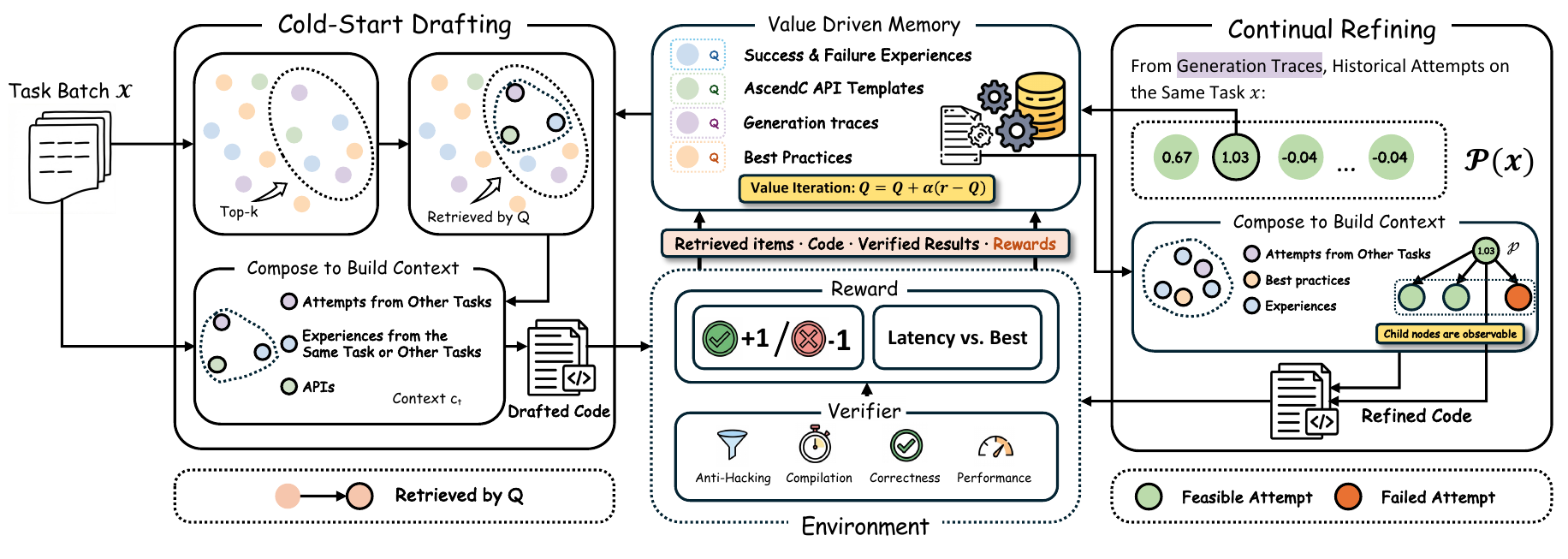}
\caption{The \ourmethod framework. \textbf{(Left) Cold-Start Drafting:} Given task batch $\mathcal{X}$, retrieves top-$k$ candidates, filters context via $Q$, and synthesizes an initial kernel. \textbf{(Center) Environment \& Memory:} A multi-gate verifier assesses generated code to yield rewards, which update $Q$ via value iteration; code and results are stored in Memory. \textbf{(Right) Continual Refining:} Exploits generation traces $\mathcal{P}(x)$ and historical attempts, including observable child nodes, to iteratively optimize for lower latency.}
\label{fig:framework}
\end{figure*}

\subsection{Problem Formulation}
\label{sec:problem_formulation}

A kernel synthesis task $x\in\mathcal{X}$ is specified by a PyTorch reference operator and metadata (e.g., input shapes and operator hyperparameters).
Given a task $x$ and retrieved context $c$, a generator $G_\theta$ samples a kernel and the goal is to generate a kernel source code $y \in \mathcal{Y}$ that satisfies functional correctness and minimizes execution latency.

We model the generation process as an M-MDP over a horizon $T$.A trajectory is defined as $\tau = (s_0, c_0, a_0, r_0, \dots, s_T)$, governed by the tuple $(\mathcal{S}, \mathcal{A}, \mathcal{M}, \mathcal{P}, \mathcal{R})$. The components are defined as follows:

\textbf{State Space ($\mathcal{S}$):}
A state $s_t$ is defined as a tuple $(x, \xi_t)$, where $x \in \mathcal{X}$ denotes the static kernel task (PyTorch operator + metadata), and $\xi_t$ represents the \emph{dynamic generation state} (e.g., current best-so-far latency or verification status).

\textbf{Action Space ($\mathcal{A}$):}
The action $a_t \in \mathcal{A}$ corresponds to a generated kernel code $y \in \mathcal{Y}$.

\textbf{Memory ($\mathcal{M}$):}
We define $\mathcal{M}_t$ as a dynamic, self-evolving memory bank. It is initialized as $\mathcal{M}_0$ comprising seed knowledge. At each step $t$, it accumulates the agent's interaction history, updating according to the rule:
\begin{equation}
\mathcal{M}_{t+1} \leftarrow \mathcal{M}_t \cup \{ (s_t, a_t, r_t) \},
\end{equation}
\textbf{Transition Dynamics ($\mathcal{P}$):}
The transition dynamics $\mathcal{P}: \mathcal{S} \times \mathcal{A} \rightarrow \Delta(\mathcal{S})$ describe the evolution of the generation process.
Since task $x$ remains invariant within an episode, $\mathcal{P}$ deterministically updates the generation state:
\begin{equation}
s_{t+1} = (x, \xi_{t+1}), \quad \xi_{t+1} = f(x, \xi_t, a_t, o_t),
\end{equation}
Here, $f$ updates the dynamic generation state by integrating the action $a_t$ and its verifier outcome $o_t$, conditioned on the task $x$ and the previous state $\xi_t$.

\textbf{Reward Function ($\mathcal{R}$):}
The environment provides a scalar feedback signal $r_t \in \mathbb{R}$ based on evaluation of the action $a_t$. 

\textbf{Policy Factorization.} To tackle this M-MDP, the agent operates via a composite policy. At each step $t$, a Retrieval Policy $\mu$ first selects a context $c_t \subset \mathcal{M}_t$ based on the current state. Conditioned on this context, the Generator Policy $G_\theta$  samples the code:
\begin{equation}
\label{eq:policy_factorization}
\pi(y_t | s_t, \mathcal{M}_t) = G_\theta(a_t | s_t, c_t) \cdot \mu(c_t | s_t, \mathcal{M}_t),
\end{equation}
Our core methodology focuses on optimizing $\mu$ via reinforcement learning to identify high-utility memory items, while $G_\theta$ leverages the pre-trained capabilities of the LLM.

\subsection{Memory Architecture and Value-Driven Retrieval}
\label{sec:memory_value}
The efficacy of the generator $G_\theta$ depends critically on the quality of the context $c_t$. We design $\mathcal{M}$ as a heterogeneous knowledge base containing: 
(i) API templates for the active backend when such documentation is available (e.g., Ascend C),
(ii) summarized success and failure experiences,
(iii) generation traces, including both draft and refined variants, and
(iv) best practices for kernel refinement.

To instantiate the policy $\mu$, we introduce \textbf{Value-Driven Retrieval}. Unlike traditional similarity-based retrieval, our approach dynamically evaluates memory item utility based on the current generation stage. For state $s$ and candidate memory item $m$, we define a Q-value function $Q_k(s, m)$ that estimates the expected benefit of including $m$ in the context at stage $k$.

For a given task $x$, let $N$ denote the final retrieval count. We first use dense retrieval to obtain a top-$K$ candidate pool $\mathcal{C}(x) \subset \mathcal{M}$, where $K=\lambda N$ and $\lambda$ is an over-retrieval multiplier. We then use stage-specific value estimates $Q_k$ to filter these top-$K$ candidates down to the final $N$ items for the context. These values reflect the agent's evolving objectives:
\begin{itemize}
\item \textbf{Drafting Stage ($Q_1$):} Estimates the likelihood that $m$ contributes to a \emph{functionally correct} kernel.
\item \textbf{Refining Stage ($Q_2$):} Estimates the contribution of a memory item $m$ to \emph{latency optimization} of the kernel, where $m$ can either be an optimization start point $p$ or auxiliary refinement items from $\mathcal{M}$.
\end{itemize}

In the upgraded system, the drafting context is assembled by a \emph{hybrid retrieval policy}: experiential memories and code traces remain value-selected, while API knowledge is retrieved through a backend-aware mixture of static infrastructure bundles, exact-name lookup from retrieved code examples, and semantic/category-based search. This separation is important in practice because API utility is largely determined by backend coverage and operator decomposition, whereas experiential memories benefit directly from online value estimation.

\textbf{Unified Value Update.}
Despite the distinct objectives, we employ a unified Monte-Carlo (MC) update rule to refine the retrieval policy $\mu$. Upon observing a reward $r_t$ (defined in subsequent sections) after using context items $c_t$, we update the Q-values for all $m \in c_t$:
\begin{equation}
  \label{eq:unified_update}
  Q(s, m) \leftarrow Q(s, m) + \alpha \cdot \big(r - Q(s, m)\big),
\end{equation}
where $\alpha$ is the step size. This update rule allows the retrieval policy $\mu$ to continuously adapt to the evolving capabilities of $G_\theta$. We provide formal guarantees on boundedness and convergence of these value estimates in Appendix~\ref{app:proofs}.

\subsection{Stage 1: Cold-Start Drafting}
\label{sec:stage1}
The objective of this stage is to obtain an initial \emph{feasible} kernel that can bootstrap subsequent refinement.
For a task $x$, we iteratively (i) retrieve a drafting context $c_t \subset \mathcal{C}(x)$ using an $\epsilon$-greedy policy over $Q_{1}$, and (ii) sample a candidate kernel $y_t \sim G_\theta(\cdot \mid x,c_t)$.

\textbf{Reward and update.}
We use a binary feasibility reward
\begin{equation}
r_{1,t}=
\begin{cases}
+1, & \text{if } g_{\text{feas}}(o_t)=1,\\
-1, & \text{otherwise},
\end{cases}
\end{equation}
where $o_t=V(x,y_t)$ and $g_{\text{feas}}$ is the combined feasibility gate (Section~\ref{sec:verifier}).
After receiving feedback, we update the values of retrieved entries $m \in c_t$ using Eq.~\ref{eq:unified_update} with $r=r_{1,t}$ and store the generated code together with verifier feedback into memory. This process repeats until a feasible kernel is found or the budget is exhausted.

\begin{table*}[!t]
\centering

\caption{Compilation Rate (CR) and Correctness (Acc) across difficulty levels, shown as \textbf{(Round 1) Final}, respectively representing the start point and the final performance. Notably, the huge gap between GPT-5.2 and other models implies that frontier LLMs with stronger in-context learning capability benefit substantially more from experience-driven methods. }
    \label{tab:main_results}

    \setlength{\tabcolsep}{4pt}
    \renewcommand{\arraystretch}{1.5}

\begin{tabular}{ll rr rr rr}
      \toprule
\multirow{2}{*}{\textbf{Model}} & \multirow{2}{*}{\textbf{Method}} & \multicolumn{2}{c}{\textbf{Level 1}} & \multicolumn{2}{c}{\textbf{Level 2}} & \multicolumn{2}{c}{\textbf{Overall}} \\
\cmidrule(lr){3-4} \cmidrule(lr){5-6} \cmidrule(lr){7-8}
& & \textbf{CR (\%)} & \textbf{Acc (\%)} & \textbf{CR (\%)} & \textbf{Acc (\%)} & \textbf{CR (\%)} & \textbf{Acc (\%)} \\
      \midrule
\multirow{3}{*}{Qwen3-Coder-30B}
& Pass@$k$ & (22.0) 30.0 & (7.0) 8.0 & (0.0) 2.0 & (0.0) 0.0 & (11.0) 16.0 & (3.5) 4.0 \\
& Refinement & (13.0) 22.0 & (2.0) 6.0 & (0.0) 1.0 & (0.0) 0.0 & (6.5) 11.5 & (1.0) 3.0 \\
& \textbf{Ours} & \textbf{(25.0) 33.0} & \textbf{(6.0) 11.0} & \textbf{(1.0) 3.0} & (0.0) 0.0 & \textbf{(13.0) 18.0} & \textbf{(3.0) 5.5} \\
      \midrule
\multirow{3}{*}{DeepSeek-V3.2}
& Pass@$k$ & (21.0) 33.0 & (7.0) 9.0 & (1.0) 13.0 & (0.0) 0.0 & (11.0) 23.0 & (3.5) 4.5 \\
& Refinement & \textbf{(16.0) 44.0} & (0.0) 12.0 & \textbf{(2.0) 26.0} & (0.0) 0.0 & \textbf{(9.0) 35.0} & (0.0) 6.0 \\
& \textbf{Ours} & (9.0) 39.0 & \textbf{(2.0) 19.0} & (1.0) 19.0 & (0.0) 0.0 & (5.0) 29.0 & \textbf{(1.0) 9.5} \\
      \midrule
\multirow{4}{*}{GPT-5.2}
& Pass@$k$ & (24.0) 36.0 & (9.0) 19.0 & (2.0) 13.0 & (1.0) 3.0 & (13.0) 24.5 & (5.0) 11.0 \\
& Refinement & (19.0) 88.0 & (7.0) 41.0 & (2.0) 55.0 & (1.0) 3.0 & (10.5) 71.5 & (4.0) 22.0 \\
& Codex & (34.0) 82.0 & (16.0) 70.0 & (16.0) 84.0 & (0.0) 22.0 & (25.0) 83.0 & (8.0) 46.0 \\
& \textbf{Ours} & \textbf{(20.0) 97.0} & \textbf{(7.0) 90.0} & \textbf{(2.0) 100.0} & \textbf{(1.0) 76.0} & \textbf{(11.0) 98.5} & \textbf{(4.0) 83.0} \\
      \bottomrule
    \end{tabular}
  \end{table*}

\subsection{Stage 2: Continual Refining}
\label{sec:stage2}
Once a feasible kernel is obtained, the focus shifts from feasibility to \emph{latency reduction}.
We maintain a set of \emph{optimization start points}  $\mathcal{P}(x)$, initialized with the successful draft from Stage~1 and augmented online as new feasible variants are discovered.
At each iteration, based on the current state $s_t$, we retrieve the available start points from the memory $\mathcal{M}$ and select a start point using $Q_2$.

With the selected start point and the current state, we then retrieve additional contextual information that contains optimization traces, best practices, and information about its observable child nodes to support the refinement process. In the upgraded system, this refinement context is further conditioned on profiler-derived bottleneck diagnoses, which are used to retrieve bottleneck-matched optimization examples and complementary high-performing variants.
Using the selected start point and the retrieved context in $c_t$, the generator samples a refined result.

\textbf{Relative reward, normalization, and update.}
To drive performance optimization, we define reward relative to the best-so-far latency $b_{t}$ tracked in $\xi_{t}$:
\begin{equation}
r_{2,t} =
\begin{cases}
-1, & \text{if } g_{\text{feas}}(o_t)=0,\\
\tanh\!\left(\log b_{t} - \log \ell_{\text{lat}}(o_t)\right), & \text{otherwise}.
\end{cases}
\end{equation}
We apply PopArt-style online normalization $\hat{r}_{2,t} = (r_{2,t}-\mu_2)/\sigma_2$ using running estimates $(\mu_2,\sigma_2)$.
We update $Q_{2}$ for both the start point $p_t$ and retrieved entities $z \in c_t$ using Eq.~\ref{eq:unified_update} with $r=\hat{r}_{2,t}$.
When a refined kernel is feasible, as indicated by $g_{\text{feas}}(o_t)=1$, we store the kernel together with verifier feedback in memory for future retrieval and add it to the start set $\mathcal{P}(x)$ to expand the refinement search space.

\subsection{Multi-gate Verification}
\label{sec:verifier}
The verifier $V$ acts as the environment interface, providing robust feedback to guide the RL process. Given a task $x$ and a generated kernel $y_t$, it returns a structured outcome
\begin{equation}
o_t = V(x, y_t) = (g_{\text{hack}}, g_{\text{comp}}, g_{\text{corr}}, \ell_{\text{lat}}),
\end{equation}
where $g_{\text{hack}}, g_{\text{comp}}, g_{\text{corr}} \in \{0,1\}$ denote the anti-hacking, compilation, and correctness gates, and $\ell_{\text{lat}} \in \mathbb{R}_{+}$ is the measured latency. A kernel is deemed feasible if and only if:
$g_{\text{feas}}(o_t) \triangleq g_{\text{hack}} \wedge g_{\text{comp}} \wedge g_{\text{corr}}$.

\textbf{Anti-hacking ($g_{\text{hack}}$).}
We implement a two-tier screening process. A rule-based filter first rejects trivial exploits (e.g., using high-level \texttt{torch} APIs or constant-folding shortcuts). Survivors undergo a model-based inspection to identify subtle harness manipulations.

\textbf{Compilation ($g_{\text{comp}}$) \& Correctness ($g_{\text{corr}}$).}
We verify successful compilation under the backend-specific toolchain, instantiated in our main study with the Ascend C toolchain. Correctness is validated by comparing outputs against the PyTorch reference: $\|\mathrm{out}_{y}(x)-\mathrm{ref}(x)\| \le \tau$. The verifier provides fine-grained feedback, including mismatch localization and shape errors (details in Appendix~\ref{app:eval_methodology}).

\textbf{Latency ($\ell_{\text{lat}}$).}
For feasible kernels, we measure on-device execution time using backend-native profiling tools. In the primary Ascend setting, we use \textbf{msprof} and report the mean wall time across 3 profiling passes (Pipe, Memory, Resource) after warm-up. In extended experiments on CUDA, the same loop is instantiated with GPU-native profiling signals.

\section{Experiment}
\label{sec:experiments}
\subsection{Experimental Setup}
\label{sec:exp_setup}
\textbf{Benchmark and Execution.}
We evaluate on L1 and L2 operators from KernelBench~\citep{ouyang2025kernel}. Since KernelBench does not natively support Ascend C, we implement a compilation, deployment, and execution pipeline that maintains full compatibility with KernelBench PyTorch references while enabling the model to generate complete Ascend operator projects.

\textbf{Budget and metric.}
We enforce a strict per-operator budget of $T=30$ iterations across all methods, encompassing both draft generation and iterative refinement. Functional correctness is verified with tolerances of $\texttt{atol}=\texttt{rtol}=10^{-2}$. Our evaluation relies on three primary metrics: (i) \textbf{Compilation Rate (CR)}, which measures the proportion of generated kernels that successfully compile, and (ii) \textbf{Correctness (Acc)}, which reports the percentage of operators for which a functionally valid solution is found within the budget.(iii) \textbf{Speedup} measures the reduction in execution latency, defined as $\text{speedup} = L_{\text{ref}} / L_{\text{opt}}$, where $L_{\text{ref}}$ and $L_{\text{opt}}$ are the latencies of the reference and optimized kernels, respectively.

\textbf{Baselines.}
We compare \ourmethod against three baseline strategies using three models: Qwen3-Coder-30B-A3B-Instruct~\citep{yang2025qwen3}, DeepSeek-V3.2~\citep{liu2025deepseek}, and GPT-5.2. Detailed configurations of these baselines can be found in Appendix~\ref{app:baselines}.
\begin{itemize}
    \item \emph{Pass@$k$}: A stateless baseline generating $K=30$ independent candidates per operator given a single demonstration.
    \item \emph{Refinement}: A stateful agentic loop that iteratively repairs compilation and correctness errors using verifier feedback. Upon finding a valid kernel, it transitions to hill-climbing for latency optimization, subject to a maximum budget of 30 iterations.
    \item \emph{Codex by OpenAI}: An autonomous agent based on GPT-5.2 with direct shell and file system access. It executes a ``try-fail-evolve'' loop, autonomously mutating the implementation based on execution logs until success or a budget of 30 verification attempts is exhausted.
\end{itemize}

\subsection{Main Results}
\label{sec:main_results}
We evaluate \ourmethod under a matched evaluation pipeline, focusing on compilation and correctness, as well as performance optimization after correctness.

\textbf{Compilation and correctness.}
Table~\ref{tab:main_results} reports compilation rate (CR) and correctness (Acc) across two difficulty levels under a fixed budget $T{=}30$.

\ourmethod achieves the strongest overall performance with GPT-5.2, reaching \textbf{98.5\%} CR and \textbf{83.0\%} Acc, substantially outperforming Codex (83.0\% CR, 46.0\% Acc) and Refinement (71.5\% CR, 22.0\% Acc).
On Level~2, \ourmethod attains near-perfect compilation (100\%) with 76\% correctness.
Despite Codex having autonomous shell and file system access, \ourmethod surpasses it by 15.5 points in CR and 37.0 points in Acc.

On weaker backbones, the improvements are more moderate.
\ourmethod achieves the highest Acc on both Qwen3-Coder-30B (5.5\% vs.\ 4.0\%) and DeepSeek-V3.2 (9.5\% vs.\ 6.0\%), with DeepSeek-V3.2 reaching 19\% correctness on Level~1---more than doubling Pass@$k$.
The Refinement baseline attains higher CR on DeepSeek-V3.2 (35.0\% vs.\ 29.0\%), suggesting that value-driven retrieval prioritizes generation quality over compilation attempts.
Critically, Level~2 Acc remains at 0\% for weaker models even when candidates compile (e.g., 19\% CR on DeepSeek-V3.2), indicating that harder operators demand stronger generator capacity.

Examining Round 1 through the final iteration reveals how effectively each method leverages the iterative process.
On GPT-5.2, \ourmethod improves CR from 11.0\% to 98.5\% and Acc from 4.0\% to 83.0\%, representing an order-of-magnitude gain.
In contrast, weaker models show limited improvement: Qwen3-Coder-30B increases Acc by +2.5 points, while DeepSeek-V3.2 improves by +8.5 points.
This disparity reveals a key insight: the in-context learning capabilities of frontier LLMs prove critical for experience-driven approaches like ours.
Crucially, it does not weaken our method's value; instead, it confirms that our agent is keeping pace with the cutting-edge advancements of base models.

\begin{figure}[!htbp]
\centering
\includegraphics[width=\linewidth]{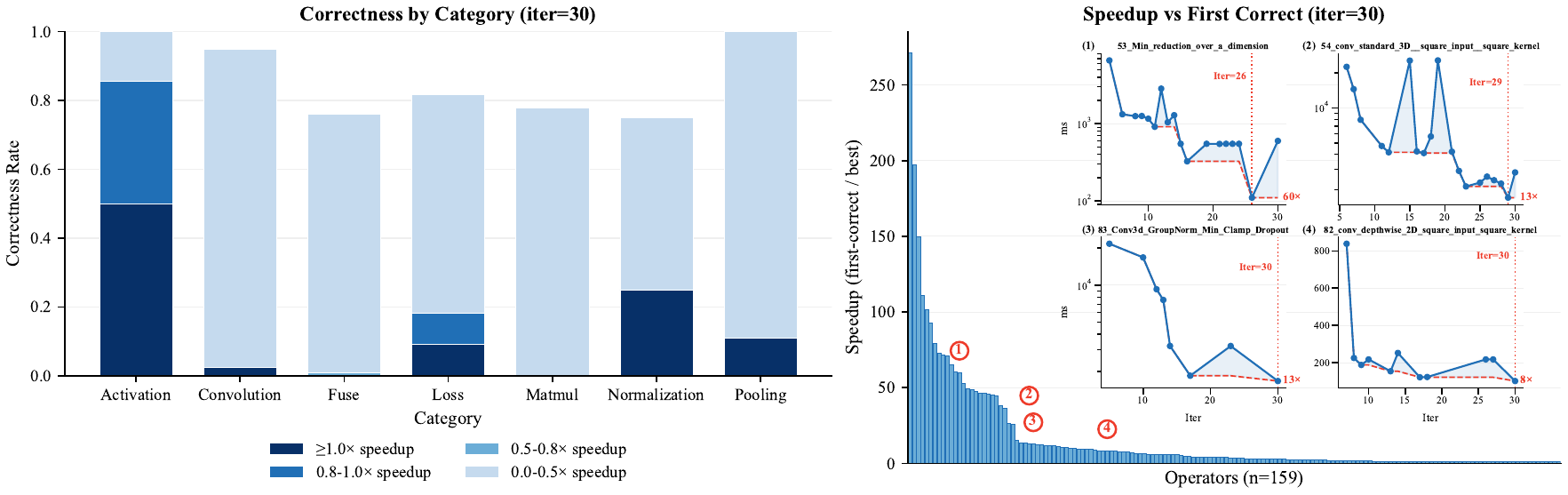}
\caption{Optimization outcomes. \textbf{(Left)} Category-level correctness and speedup distribution at budget $T{=}30$; color segments show the fraction of correct kernels in each speedup tier relative to Torch-NPU. \textbf{(Right)} Within-operator speedup achieved by iterative refinement across 159 operators with $\geq$1 valid optimization candidate beyond the initial correct draft; inset panels detail representative optimization trajectories.}
\label{fig:optimization_outcomes}
\end{figure}

\textbf{Optimization gains: within-operator speedup.}
Conditioned on reaching a correct draft, the refining stage further reduces latency.
For each solved operator, we compare the initial draft, defined as the \emph{first feasible} candidate, to the \emph{best} candidate found within the remaining budget.
This yields a median speedup of \textbf{3.60}$\times$, with an interquartile range of \textbf{1.38}--\textbf{10.05}$\times$.
Although many operators remain slower than Torch-NPU (Figure~\ref{fig:optimization_outcomes}), consistent within-operator gains indicate that the refinement process continues to improve performance beyond correctness. 

Figure~\ref{fig:optimization_outcomes} quantifies these gains across 159 operators with at least one valid optimization candidate beyond the initial correct draft.
The distribution is long-tailed: while many operators exhibit modest improvements ($s \approx 1$--$2\times$), a substantial subset benefits dramatically from continued optimization, with top performers achieving more than $200\times$ speedup over their first correct version.

Inset trajectories for four representative operators confirm that these gains emerge from systematic, incremental improvements across multiple iterations, rather than from single fortuitous generations.

\subsection{Generalization of Value-Driven Memory}
\label{sec:transfer_stream}
A core motivation for our memory design is \emph{reusability}: high-utility past experiences should accelerate learning on subsequent ones. We verify this hypothesis by evaluating transfer across difficulty levels and generator backbones.

\begin{figure}[!htbp]
\centering
\includegraphics[width=\linewidth]{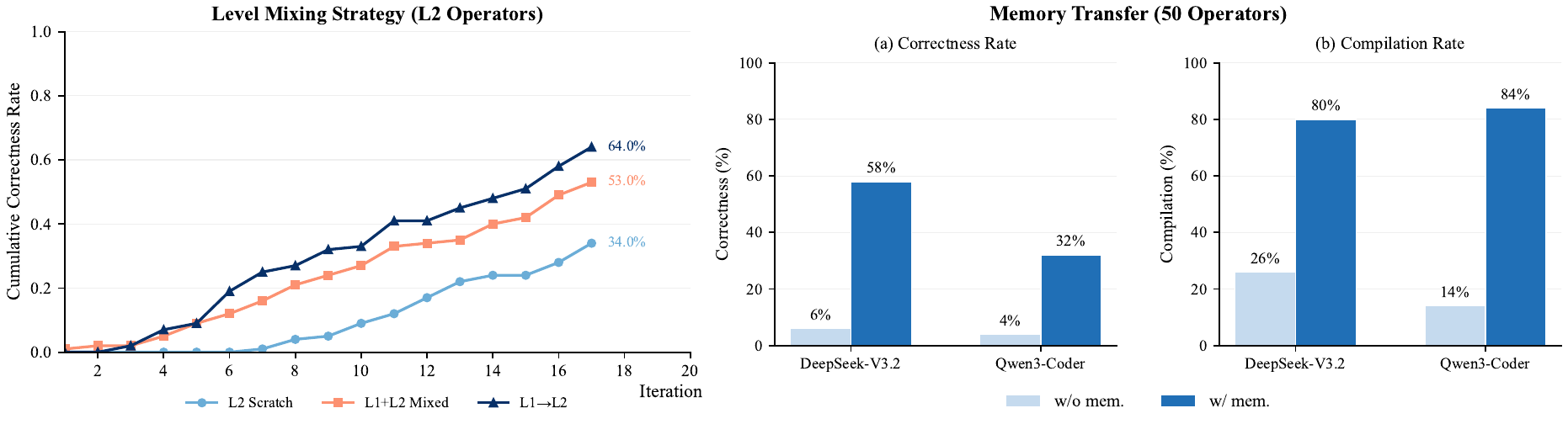}
\caption{Transfer and generalization. \textbf{(Left)} Transfer across difficulty levels: cumulative success rate on L2 under different stream compositions. \textbf{(Right)} Transfer across generator backbones: performance on held-out operators when reusing memory built with GPT-5.2.}
\label{fig:transfer_generalization}
\end{figure}

\textbf{Transfer across difficulty levels.}
We study whether memory accumulated on easier L1 operators transfers to harder L2 operators.
We consider three setups:
\begin{itemize}
  \item \emph{L2 Scratch}: agent iterates from scratch on L2 operators.
  \item \emph{L1+L2 Mixed}: the agent iterates from scratch on a mixed operator set containing both L1 and L2.
  \item \emph{L1 $\rightarrow$ L2}: the agent first iterates on L1, then continues iterating on the L2 operator set initialized with the resulting L1 memory.
\end{itemize}

In Figure~\ref{fig:transfer_generalization} and Table~\ref{tab:transfer_summary}, the \emph{L1 $\rightarrow$ L2} stream exhibits the fastest warm-up and highest final performance. By iteration $t=17$, it achieves 64\% L2 correctness, outperforming \emph{L1+L2 Mixed} (53\%) by 11\% and \emph{L2 Scratch} (34\%) by 30\%. Crucially, the transfer allows the agent to solve its first L2 operator four iterations earlier than the scratch baseline. This confirms that foundational patterns learned from simpler tasks effectively bootstrap progress on harder problems.

\begin{table}[ht]
\centering
\caption{Cross-level transfer summary on L2 at final iteration.}
\label{tab:transfer_summary}
\setlength{\tabcolsep}{12pt}
\begin{tabular}{lcc}
\toprule
\textbf{Setup} & \textbf{CR (\%)} & \textbf{Acc (\%)} \\
\midrule
L2 Scratch & 88.0 & 34.0 \\
L1+L2 Mixed & 98.0 & 53.0 \\
L1$\rightarrow$L2 & 97.0 & 64.0 \\
\bottomrule
\end{tabular}
\end{table}

\textbf{Transfer across generator backbones.}
We further assess whether memory constructed by a strong model (GPT-5.2) can improve the performance of weaker backbones (DeepSeek-V3.2, Qwen3-Coder-30B). We evaluate on a held-out set of 50 operators (30 L1, 20 L2), initializing the agent with a filtered GPT-5.2 memory bank where traces from the test operators are excluded to prevent leakage.

Figure~\ref{fig:transfer_generalization} (right) shows that the learned memory transfers well across generator backbones.
For DeepSeek, adding memory improves compilation from 26\% to 80\% and correctness from 6\% to 58\%.
For Qwen, memory yields a similarly large compilation gain (14\%$\rightarrow$84\%) with a smaller but substantial correctness gain (4\%$\rightarrow$32\%).

Overall, memory appears to provide backbone-agnostic operator constraints and debugging cues that greatly reduce non-compiling attempts, while the remaining compilation--correctness gap (especially for Qwen) suggests semantic validity remains the dominant bottleneck.

\FloatBarrier
\subsection{Beyond KernelBench and CANN}
\label{sec:scaling_out}
The main benchmark in this paper is Ascend C KernelBench, but an important question is whether the learned memory and refinement policy continue to help on workloads that fall outside this training distribution. To test this, we evaluate EvoKernel on the Attention Set operator suites and on \emph{mHC} kernels~\citep{xie2025mhc} derived from recent DeepSeek architectures.

\begin{table*}[!htbp]
\centering
\caption{Initial scaling-out results beyond the main KernelBench study.}
\label{tab:scaling_out}
\setlength{\tabcolsep}{8pt}
\renewcommand{\arraystretch}{1.15}
\begin{tabular}{lccccl}
\toprule
\textbf{Workload} & \textbf{Platform} & \textbf{\# Ops} & \textbf{CR (\%)} & \textbf{Acc (\%)} & \textbf{Fast$_1$ Ratio (\%)} \\
\midrule
Attention Set & CUDA & 70 & 100.0 & 97.1 & 72.1 \\
Attention Set & Ascend & 70 & 100.0 & 78.6 & 21.8 \\
KernelBench & CUDA & 250 & 100.0 & 100.0 & 68.0 \\
mHC Kernels (DeepSeek) & Ascend & 15 & 86.7 & 66.7 & 60.0 \\
\bottomrule
\end{tabular}
\end{table*}

\textbf{Attention Set operators.}
On CUDA, EvoKernel scales cleanly from the KernelBench-style setting to the Attention Set workloads. On the 70-operator Attention Set (excluding non-attention operators such as GEMM, RoPE, and Router), the system reaches 100\% compilation and 97.1\% correctness after 30 outer iterations. On the same Attention Set on CUDA with KernelBench operators included, \ourmethod achieves 100\% compilation and 100\% correctness on all 250 operators. These results indicate that the memory mechanism transfers beyond the operator families emphasized in the main benchmark and remains effective on more application-driven kernels. Figure~\ref{fig:attention_cuda_combined} shows the optimization timeline and performance comparison for the CUDA Attention Set.

\textbf{Ascend Attention Set and DeepSeek mHC kernels.}
More importantly for the cold-start setting emphasized in this paper, the same methodology also transfers to new Ascend C workloads outside the original KernelBench distribution. On the 70-operator Ascend Attention Set, EvoKernel reaches 100.0\% compilation and 78.6\% correctness after 30 iterations. We further evaluate on 15 mHC kernels targeting a recent DeepSeek architectural motif on Ascend (CANN~8.5.0). EvoKernel obtains 10 correct implementations, and 6 of these outperform the PyTorch baseline. Representative wins include \texttt{SinkhornKnopp} with 41.96$\times$ speedup, \texttt{OrthostochasticProject} with 2.94$\times$, and \texttt{MhcPostBlock} with 2.88$\times$. Figure~\ref{fig:mhc_combined} shows the optimization timeline and per-operator performance for all 15 mHC kernels over 30 iterations (merged across three experiment series).

\begin{figure*}[!htbp]
\centering
\includegraphics[width=\textwidth]{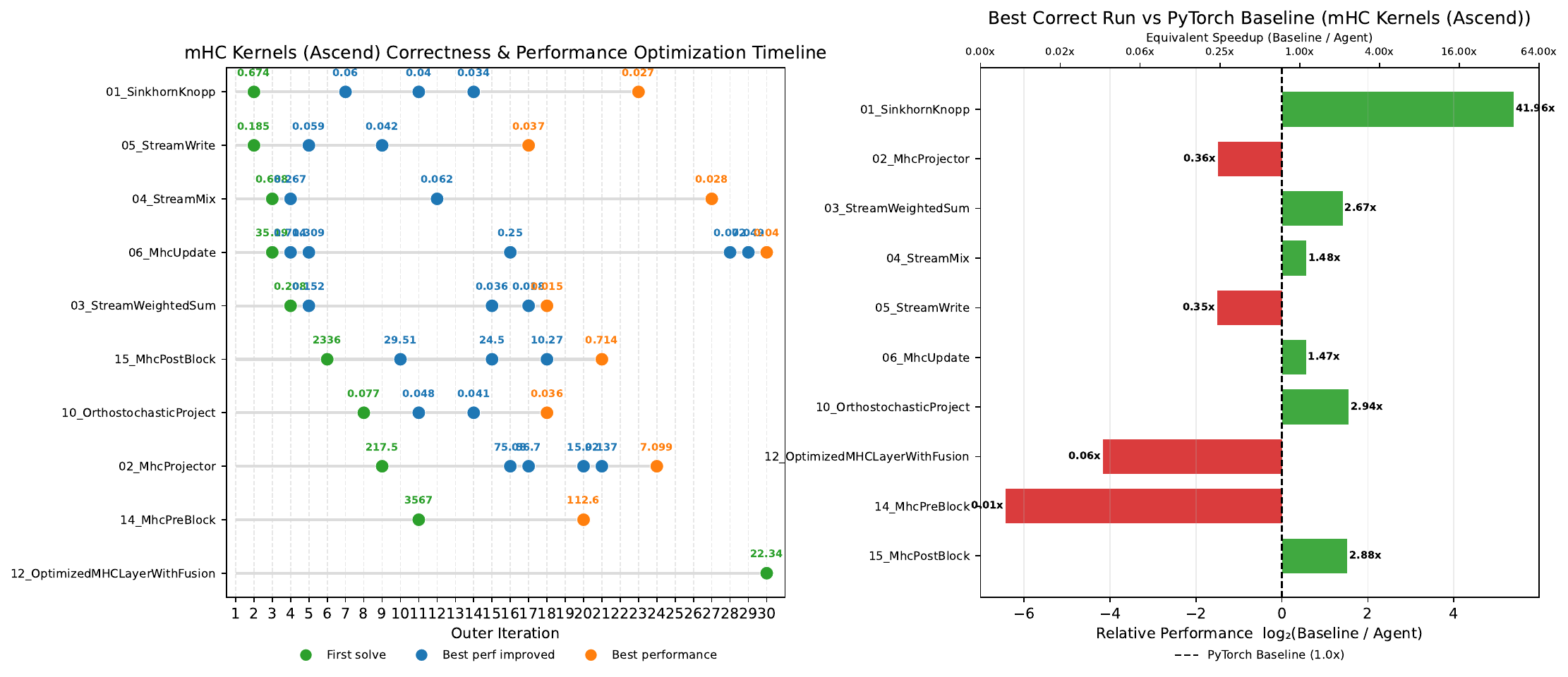}
\caption{mHC Kernels (Ascend): Optimization timeline and performance vs.\ Torch-NPU baseline for 15 DeepSeek mHC operators over 30 iterations (merged across three experiment series). \textbf{(Left)} Correctness and performance optimization timeline. \textbf{(Right)} Best correct run vs.\ baseline in log$_2$ speedup.}
\label{fig:mhc_combined}
\end{figure*}

Taken together, these scaling-out results suggest that EvoKernel is not simply memorizing KernelBench operator templates. Instead, the framework appears able to reuse memory and profiling-guided refinement to adapt to both new operator families and new architectural motifs, while still preserving the paper's primary emphasis on data-scarce Ascend C kernel generation.

\FloatBarrier
\subsection{Ablations}
\label{sec:ablations}
\subsubsection{Value-Driven versus Heuristic-Driven Retrieval}
\label{sec:ablation_qvalue}
We assess the impact of learned value estimates by comparing our full value-driven pipeline against a heuristic-driven variant. Both settings use the \emph{L1 $\rightarrow$ L2} transfer protocol (Section~\ref{sec:transfer_stream}) and run for 30 L2 iterations per operator, inheriting the same L1 memory. The only difference lies in the selection mechanism:
\begin{itemize}
    \item \textbf{Value-Driven (Ours):} Selects context and optimization start points using $\epsilon$-greedy over learned $Q$-values.
    \item \textbf{Heuristic-Driven:} Selects context based solely on semantic similarity and chooses optimization start points based on the highest historical performance.
\end{itemize}

Figure~\ref{fig:retrieval_ablations} (left) tracks cumulative correctness and compilation rates. While both methods perform similarly in the early stages (reaching 48\% correctness by iteration 14), the value-driven approach diverges significantly thereafter. By iteration 30, it achieves 77\% correctness and 100\% compilation, compared to 67\% and 97\% for the heuristic baseline. 
This indicates that while heuristics suffice for initial bootstrapping, learned value estimates provide a crucial exploitation signal for solving the long tail of difficult operators.

\begin{figure}[!htbp]
\centering
\includegraphics[width=\linewidth]{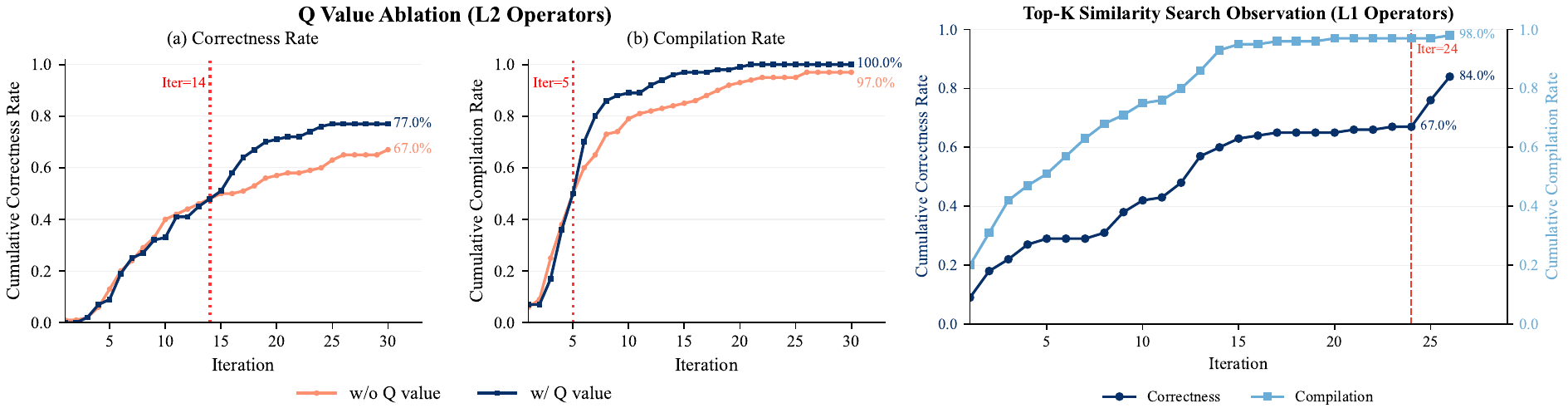}
\caption{Retrieval ablations. \textbf{(Left)} Value-driven vs.\ heuristic retrieval on L2 operators (same L1 memory and $\epsilon$-greedy schedule). \textbf{(Right)} Effect of increasing retrieval pool size $K$ at iteration 24; cumulative correctness and compilation rates on L1 operators.}
\label{fig:retrieval_ablations}
\end{figure}

\subsubsection{Multi-Task Memory Sharing versus Per-Task Refinement}
To isolate the contribution of cross-task memory sharing, we compare \ourmethod against the \emph{Refinement} baseline under identical per-operator iteration budgets (Table~\ref{tab:main_results}).
Refinement can be viewed as a degenerate instance of our framework: restricting the memory bank to a single operator eliminates cross-task retrieval, reducing the agent to iterative self-refinement.
This controlled ablation thus directly quantifies the benefit of a \emph{global}, shared memory bank over per-task isolated iteration.

Results reveal that cross-task sharing yields substantial gains, particularly on Level~2 operators.
With GPT-5.2, \ourmethod raises the Level~2 compilation rate from 55.0\% to 100.0\% and accuracy from 3.0\% to 76.0\%. Level~1 exhibits more moderate improvements (+9\,pp CR, +49\,pp Acc).
These findings indicate that, although within-operator refinement provides a useful signal, the ability to transfer experience across tasks confers additional, complementary benefits that isolated iteration cannot achieve.

\FloatBarrier
\subsection{Discussion}
\label{sec:discussion}
\textbf{Explicit versus Emergent Curricula.}
Our results demonstrate that value-driven memory induces \emph{adaptive curriculum learning} without explicit task ordering. 
When we impose an explicit L1$\rightarrow$L2 curriculum (Table~\ref{tab:transfer_summary}), the agent benefits from a warm start, as L1 memory acts as \emph{foundational scaffolding} that accelerates early L2 progress despite the complexity gap. 
Crucially, however, even under a \emph{L1+L2 Mixed} setting with no prescribed ordering, the retrieval policy autonomously reconstructs a \emph{soft curriculum}. 
Figure~\ref{fig:operator_dependency} exemplifies this emergent behavior for \texttt{36\_RMSNorm\_} within the mixed setting: the agent first solves simpler operators, which then serve as retrieved references to facilitate the solution of harder ones, naturally forming a dependency chain without manual intervention.

\textbf{Scaling to out-of-distribution workloads.}
The additional results in Table~\ref{tab:scaling_out} strengthen this interpretation. The framework transfers not only across difficulty levels within KernelBench, but also to workloads that differ materially from the main training distribution: the Attention Set and recent DeepSeek mHC kernels. In particular, the Ascend mHC results indicate that the system can reuse accumulated memories to tackle new architectural motifs rather than only variants of benchmark operators. The Ascend Attention Set results point in the same direction, with 78.6\% correctness on 70 operators, suggesting that the agent's benefit persists even when the workload shifts toward application-driven kernels.

  \begin{figure}[!htbp]
  \centering
  \includegraphics[width=0.68\linewidth]{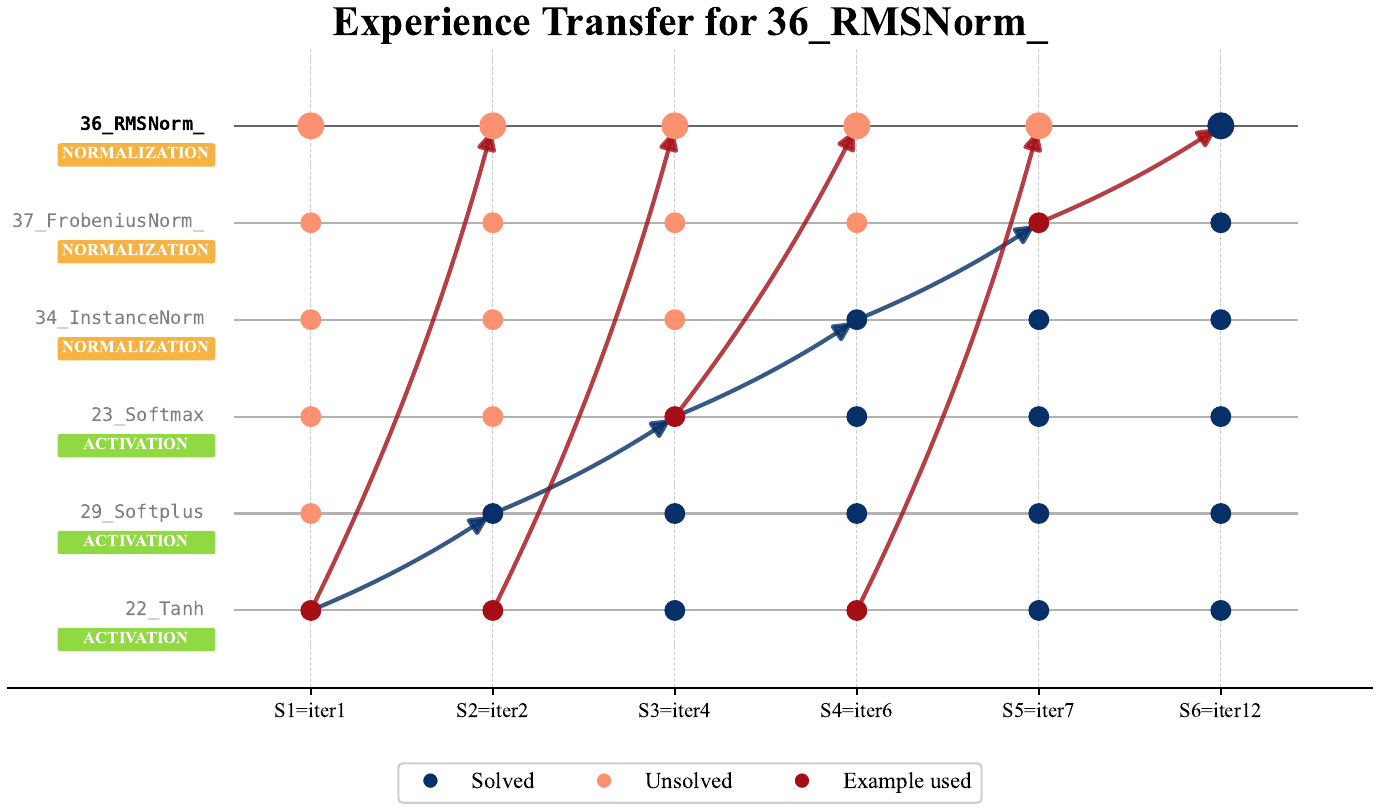}
  \caption{Experience transfer dependency graph of \texttt{36\_RMSNorm\_}. Arrows trace causal references at first-solve iterations, revealing an emergent curriculum from simple to complex operators.}
  \label{fig:operator_dependency}
  \end{figure}

  \textbf{Why value-driven memory outperforms stateless baselines.}
  Pass@$k$ sampling treats each generation independently, forfeiting any cross-attempt learning. Iterative refinement (e.g., Codex) accumulates feedback within a single operator but discards it afterward, preventing cross-operator transfer. In contrast, our approach persists and values experiences across both attempts and tasks, enabling the agent to bootstrap harder problems from easier ones and to amortize debugging effort across the entire operator population.

  \textbf{Impact of candidate pool size.}
  In our experiments, the candidate pool size $|\mathcal{C}(x)|$ is controlled by a multiplier $\lambda$ applied to the final retrieval count $N$.
  Smaller candidate pools risk missing valuable context, while larger ones may introduce noise and dilute the signal from high-value entries. Initially, we set $\lambda=2$, resulting in a convergence point with 67\% correctness. Upon increasing $\lambda$ by a factor of 15, correctness improved sharply to 84\% by iteration 26. This suggests that dynamically expanding the candidate pool during training allows the Q-value policy to discover previously overlooked high-utility entries. The optimal multiplier remains an important area for future research, as there is likely a sweet spot that balances coverage and efficiency. In our experiments, we observed that gradually increasing $\lambda$ allowed for a controlled replacement of context, ultimately improving model performance.

\FloatBarrier
\section{Conclusion and Future Work}

We presented \ourmethod, a value-driven memory agent addressing cold-start kernel synthesis by learning stage-specific Q-values for retrieval over a self-evolving memory bank. A central insight is that frontier LLMs have enhanced \emph{in-context learning} capabilities, enabling effective generalization from retrieved demonstrations even in cold-start kernel synthesis scenarios. This emergent ability makes memory-based, non-parametric approaches practically viable. Our additional scaling results on the Attention Set and recent DeepSeek mHC kernels further suggest that the learned memory is not confined to the original KernelBench distribution. More broadly, the value-driven memory paradigm may benefit other cold-start domains with binary verification signals, and we anticipate that as LLMs continue to improve, memory-augmented approaches will enable autonomous mastery of an ever-wider range of specialized tasks.
Beyond technical gains, these results suggest value-driven memory can democratize data-scarce programming expertise (e.g., NPU kernel synthesis), helping bridge expert shortages as hardware diversifies and pointing toward AI systems that adapt to new domains with minimal data.
Potential future work includes extending the framework to other emerging DSLs to verify cross-architecture universality, exploring knowledge distillation to reduce reliance on large commercial models, and incorporating denser reward signals to improve sample efficiency.

\newpage
\bibliographystyle{iclr2026_conference}
\bibliography{main}

\newpage
\appendix

  \section{Proofs for Value Update Stability and Convergence}
  \label{app:proofs}

  This appendix establishes theoretical guarantees for the value-driven memory system introduced in Section~\ref{sec:memory_value}. We prove three results: (1) boundedness of value iterates under bounded rewards, (2) stability of online reward normalization, and (3) convergence of the bandit-style update rule. Together, these lemmas ensure that the retrieval policy remains well-behaved throughout the agent's lifetime.

  \subsection{Notation and Setup}

  Fix a memory entry $i$ and a stage $s \in \{\texttt{draft}, \texttt{optimize}\}$. Each time entry $i$ is retrieved, the system observes a scalar reward $R_t$. The \emph{bandit-style} value update is
  \begin{equation}
  Q_{t+1} = Q_t + \alpha_t (R_t - Q_t) = (1 - \alpha_t) Q_t + \alpha_t R_t, \quad \alpha_t \in (0, 1].
  \label{eq:app_bandit_update}
  \end{equation}
  This is the standard incremental mean estimator used throughout reinforcement learning~\citep{sutton2018reinforcement}.

  \textbf{Reward definitions by stage.}
  (i)~\textbf{Draft stage:} Binary reward $R_t \in \{+1, -1\}$ based on feasibility.
  (ii)~\textbf{Optimize stage:} Given speedup ratio $\rho_t > 0$, the raw reward is $r_{\text{raw},t} = \tanh(\log \rho_t) \in (-1, 1)$. Optionally, we apply z-score normalization: $R_t = (r_{\text{raw},t} - \mu_{t-1}) / \sigma_{t-1}$.

  \subsection{Boundedness of Value Iterates}

  \begin{lemma}[Bounded Rewards Imply Bounded Values]
  \label{lem:boundedness}
  Suppose $|R_t| \le R_{\max}$ for all $t$ almost surely, and $\alpha_t \in (0, 1]$. If $Q_0 \in [-R_{\max}, R_{\max}]$, then $Q_t \in [-R_{\max}, R_{\max}]$ for all $t$.
  \end{lemma}

  \begin{proof}
  By induction. The update $Q_{t+1} = (1 - \alpha_t) Q_t + \alpha_t R_t$ is a convex combination of $Q_t$ and $R_t$. If both lie in $[-R_{\max}, R_{\max}]$, so does $Q_{t+1}$.
  \end{proof}

  \begin{corollary}[Boundedness of Raw Optimization Reward]
  \label{cor:tanh_bounded}
  For any $\rho_t > 0$, we have $r_{\text{raw},t} = \tanh(\log \rho_t) \in (-1, 1)$.
  \end{corollary}

  \begin{proof}
  Since $\rho_t > 0$, $\log \rho_t \in \mathbb{R}$, and $\tanh: \mathbb{R} \to (-1, 1)$.
  \end{proof}

  \begin{remark}[Z-Score Normalization Requires Safeguards]
  \label{rem:zscore}
  The z-score transformation $R_t = (r_{\text{raw},t} - \mu_{t-1}) / \sigma_{t-1}$ can be unbounded when $\sigma_{t-1} \to 0$. We ensure boundedness via either: (i) a variance floor $\hat{\sigma}_{t-1} := \max\{\sigma_{t-1}, \sigma_{\min}\}$, yielding $|R_t| \le 2/\sigma_{\min}$; or (ii) output clipping $R_t := \mathrm{clip}(R_t; -B, B)$.
  \end{remark}

  \begin{remark}[Error Clipping Alone Is Insufficient]
  \label{rem:td_clip}
  An alternative update $Q_{t+1} = Q_t + \alpha \cdot \mathrm{clip}(R_t - Q_t; -C, C)$ bounds the per-step change but not the iterates themselves. If $R_t \equiv M \gg Q_0$, then $Q_t = Q_0 + t\alpha C \to \infty$. Hence, reward boundedness (Lemma~\ref{lem:boundedness}) is essential.
  \end{remark}

  \subsection{Stability of Online Normalization}

  \begin{lemma}[Convergence of Running Statistics]
  \label{lem:popart}
  Let $\{r_{\text{raw},t}\}_{t \ge 1}$ be a strictly stationary ergodic process with $\mathbb{E}[r_{\text{raw},1}^2] < \infty$ and $\mathrm{Var}(r_{\text{raw},1}) = \sigma^2 > 0$. Define
  \begin{equation}
  \mu_t := \frac{1}{t} \sum_{k=1}^{t} r_{\text{raw},k}, \quad \sigma_t := \sqrt{\frac{1}{t} \sum_{k=1}^{t} (r_{\text{raw},k} - \mu_t)^2}.
  \end{equation}
  Then $\mu_t \to \mu := \mathbb{E}[r_{\text{raw},1}]$ and $\sigma_t \to \sigma$ almost surely. Moreover, the normalization map $f_t(r) := (r - \mu_t)/\sigma_t$ converges uniformly on bounded sets to $f_\infty(r) := (r - \mu)/\sigma$.
  \end{lemma}

  \begin{proof}
  By the ergodic theorem, $\mu_t \to \mu$ a.s. Writing $\sigma_t^2 = \frac{1}{t}\sum_k r_{\text{raw},k}^2 - \mu_t^2$ and applying ergodicity to both terms gives $\sigma_t^2 \to \sigma^2$ a.s. Continuity of $\sqrt{\cdot}$ on $(0,\infty)$ yields $\sigma_t \to \sigma$. For uniform convergence on a bounded set $J$:
  \[
  |f_t(r) - f_\infty(r)| \le \frac{|\mu - \mu_t|}{\sigma_t} + |r - \mu| \cdot \left|\frac{1}{\sigma_t} - \frac{1}{\sigma}\right| \to 0
  \]
  uniformly on $J$ since $\sigma_t \to \sigma > 0$.
  \end{proof}

  \begin{remark}[Relation to PopArt]
  PopArt~\citep{hessel2019multi} rescales network outputs when $(\mu, \sigma)$ change to preserve unnormalized predictions. Our scheme omits this rescaling; Lemma~\ref{lem:popart} shows the weaker but sufficient result that the normalization map stabilizes asymptotically.
  \end{remark}

  \subsection{Convergence of the Bandit Update}

  We analyze two regimes: constant step size (tracking) and decreasing step size (convergence).

  \begin{lemma}[Constant Step Size: EMA Dynamics]
  \label{lem:ema}
  Let $\{R_t\}$ be i.i.d.\ with mean $\mu$ and variance $\sigma_R^2 < \infty$. Under constant $\alpha \in (0,1)$:
  \begin{enumerate}[label=(\roman*), leftmargin=1.5em]
      \item $\mathbb{E}[Q_t] \to \mu$ as $t \to \infty$.
      \item $\mathrm{Var}(Q_t) \to \frac{\alpha}{2-\alpha} \sigma_R^2 = O(\alpha)$ for small $\alpha$.
      \item $\{Q_t\}$ converges in distribution to a unique stationary distribution centered at $\mu$.
  \end{enumerate}
  \end{lemma}

  \begin{proof}
  Unrolling the recursion: $Q_t = (1-\alpha)^t Q_0 + \alpha \sum_{k=0}^{t-1} (1-\alpha)^{t-1-k} R_k$.
  \begin{itemize}[leftmargin=1.5em]
      \item \textbf{Mean:} $\mathbb{E}[Q_t] = (1-\alpha)^t Q_0 + \mu(1 - (1-\alpha)^t) \to \mu$.
      \item \textbf{Variance:} $\mathrm{Var}(Q_t) = (1-\alpha)^{2t} \mathrm{Var}(Q_0) + \alpha^2 \sigma_R^2 \sum_{j=0}^{t-1} (1-\alpha)^{2j} \to \frac{\alpha}{2-\alpha}\sigma_R^2$.
      \item \textbf{Distribution:} The recursion defines an affine iterated function system with contraction $(1-\alpha) < 1$, implying geometric ergodicity~\citep{sutton2018reinforcement}.
  \end{itemize}
  \end{proof}

  \begin{lemma}[Decreasing Step Size: Almost Sure Convergence]
  \label{lem:robbins}
  Assume $|R_t| \le R_{\max}$ a.s., $\mathbb{E}[R_t \mid \mathcal{F}_t] = \mu$, and $\alpha_t$ satisfies the Robbins-Monro conditions: $\sum_t \alpha_t = \infty$ and $\sum_t \alpha_t^2 < \infty$. Then $Q_t \to \mu$ almost surely.
  \end{lemma}

  \begin{proof}
  Define $e_t := Q_t - \mu$ and $\xi_{t+1} := R_t - \mu$. Then $e_{t+1} = (1-\alpha_t) e_t + \alpha_t \xi_{t+1}$. Let $V_t := e_t^2$. By direct computation:
  \[
  \mathbb{E}[V_{t+1} \mid \mathcal{F}_t] \le (1-\alpha_t)^2 V_t + \alpha_t^2 \sigma_\xi^2 \le V_t - \alpha_t V_t + \alpha_t^2 \sigma_\xi^2.
  \]
  By the Robbins-Siegmund theorem~\citep{robbins1971convergence}, $V_t$ converges a.s.\ and $\sum_t \alpha_t V_t < \infty$. Since $\sum_t \alpha_t = \infty$, we must have $V_t \to 0$ a.s., hence $Q_t \to \mu$.
  \end{proof}

  \subsection{Summary}

  The three results work in concert: Lemma~\ref{lem:boundedness} ensures value iterates remain in a safe range when rewards are bounded (which Corollary~\ref{cor:tanh_bounded} and Remark~\ref{rem:zscore} guarantee for our reward definitions). Lemma~\ref{lem:popart} ensures the normalization map stabilizes over time. Finally, Lemmas~\ref{lem:ema} and~\ref{lem:robbins} establish that the value estimates track (constant $\alpha$) or converge to (decreasing $\alpha_t$) the true expected utility. Together, these guarantees ensure stable, well-behaved retrieval throughout the agent's lifetime.

\begin{table}[h!]
\caption{Operators where \ourmethod outperforms Torch-NPU in latency.}
\label{tab:fast_ops}
\centering
\begin{tabular}{lrrr}
\toprule
Operator & Torch-NPU ms & \ourmethod ms & Speedup \\
\midrule
32\_HardTanh & 23.873 & 4.199 & 5.69$\times$ \\
30\_Softsign & 34.756 & 9.598 & 3.62$\times$ \\
45\_Average\_Pooling\_2D & 3814.723 & 1725.443 & 2.21$\times$ \\
20\_LeakyReLU & 9.525 & 9.511 & 1.00$\times$ \\
\bottomrule
\end{tabular}
\end{table}
\section{Operator-Level Performance Results}
\label{app:full_results}

This appendix provides select operator examples demonstrating performance comparisons.
To contextualize absolute performance, we normalize latency by Torch-NPU and compare against other Ascend C approaches under the same MultiKernelBench harness. Table~\ref{tab:fast_ops} lists example operators where \ourmethod outperforms Torch-NPU.

\begin{figure*}[!htbp]
\centering
\includegraphics[width=\textwidth]{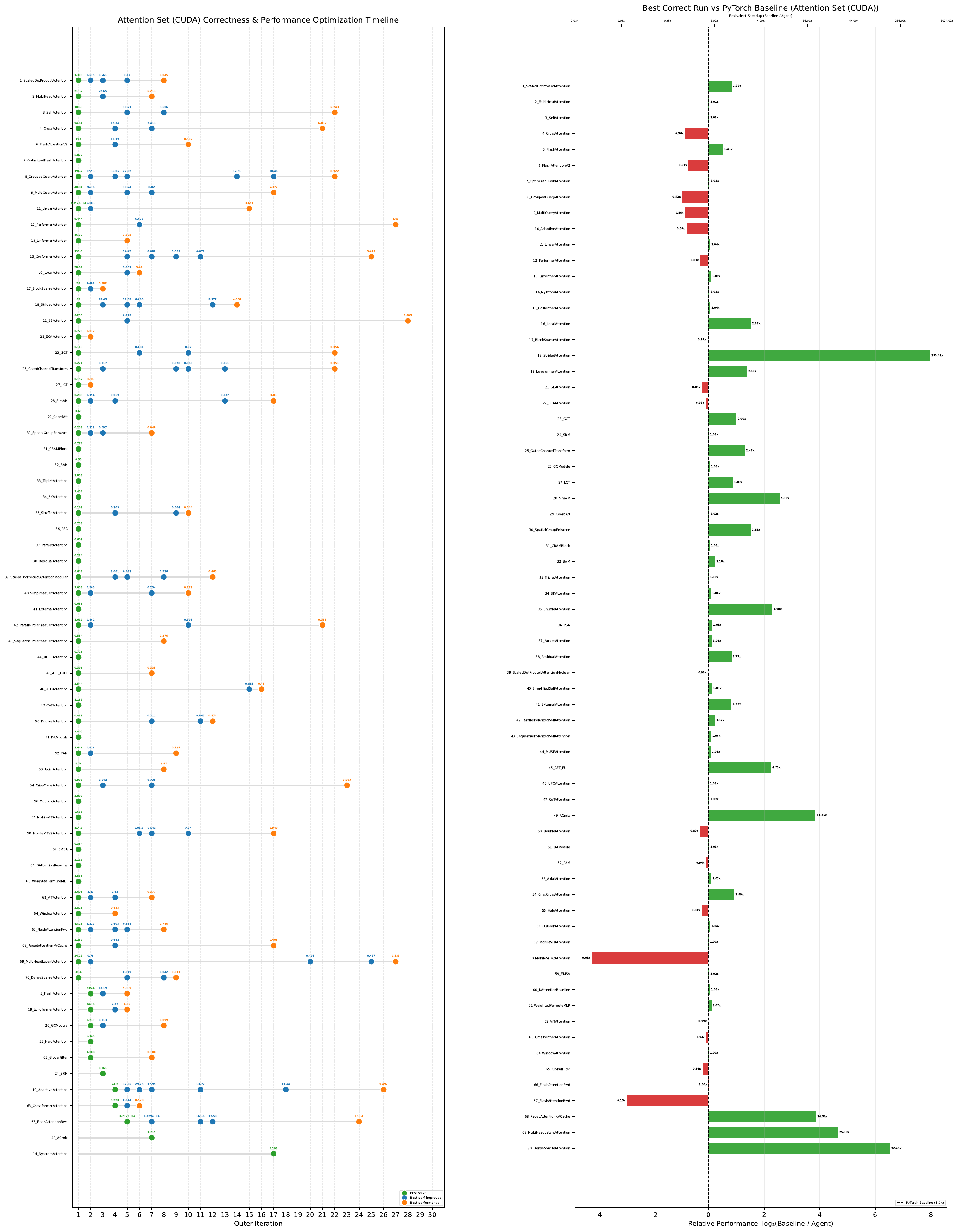}
\caption{Attention Set (CUDA): Optimization timeline and performance vs.\ PyTorch baseline for 70 attention operators over 30 iterations. \textbf{(Left)} Correctness and performance optimization timeline. Green dots denote first-solve events; blue dots denote significant ($>$10\%) performance improvements; orange dots denote the best performance achieved. \textbf{(Right)} Best correct run vs.\ baseline in log$_2$ speedup.}
\label{fig:attention_cuda_combined}
\end{figure*}

  \section{Verification Stage ``Anti-Hacking'' Screening}
  \label{app:anti_hacking}

  In the context of this work, ``anti-hacking'' refers to the architectural enforcement of the Ascend C programming paradigm. It is designed to prevent a generated solution from bypassing the intended custom operator path by re-implementing semantics in Python (within \texttt{model\_src}) or in the PyTorch binding glue (\texttt{python\_bind\_src}), rather than putting the computational logic into the Ascend C kernel (\texttt{kernel\_src}) and host tiling code.

  The verification subsystem implements this as a two-layer audit:

  \begin{enumerate}
      \item \textbf{Rule-based screening (Static/Deterministic):} Hard rules that reject common ``semantic bypass'' patterns.
      \item \textbf{Model-based screening (LLM Auditor):} A prompt-driven judgment of ``architectural integrity'' that detects subtle bypass patterns not covered by static rules.
  \end{enumerate}

  This screening acts as a strict gate: failing it short-circuits the pipeline, preventing compilation or runtime evaluation.

  \subsection{Rule-Based Screening}
  The static analyzer enforces three primary constraints:

  \textbf{1. Kernel Dispatch Requirement.}
  The binding code \texttt{python\_bind\_src} need explicitly invoke the kernel execution command \texttt{EXEC\_NPU\_CMD}. The verifier scans the binding source for this substring; its absence indicates that the operator either performs no computation or bypasses the NPU dispatch entirely.

  \textbf{2. Binding Logic Restrictions.}
  The C++ binding implementation is restricted to allocation and dispatch duties. The rule checker extracts the function body registered via \texttt{PYBIND11\_MODULE} and scans for forbidden calls to the \texttt{at::} or \texttt{torch::} namespaces.
  \begin{itemize}
      \item \textbf{Allowed:} Tensor allocation functions (e.g., \texttt{at::empty}, \texttt{at::zeros}, \texttt{at::empty\_like}).
      \item \textbf{Forbidden:} Any computational operators (e.g., \texttt{at::add}, \texttt{at::matmul}).
  \end{itemize}
  This rule guarantees that the binding layer does not perform the heavy lifting using CPU-side PyTorch reference implementations.

  \textbf{3. Model Architecture Compliance.}
  The Python invocation layer \texttt{model\_src} must define a class \texttt{ModelNew} that inherits from \texttt{torch.nn.Module}. A simplified Abstract Syntax Tree (AST) analysis enforces that:
  \begin{itemize}
      \item The \texttt{forward} method does not directly call prohibited computations (e.g., \texttt{torch.matmul}, \texttt{torch.add}) or invoke standard \texttt{torch.nn} layers created in \texttt{\_\_init\_\_}.
      \item The module must import and call the generated \texttt{custom\_ops\_lib}, ensuring the computation is delegated to the C++ binding and, by extension, the Ascend C kernel.
  \end{itemize}

  \textbf{Example Violation.}
  The following \texttt{model\_src} is rejected because it directly invokes a \texttt{torch.nn} layer (\colorbox{red!20}{\texttt{self.conv()}}) instead of delegating all computation to \texttt{custom\_ops\_lib}:

  \begin{tcolorbox}[colback=gray!5, colframe=red!50, boxrule=0.5pt, arc=2pt, title={Rejected \texttt{model\_src} -- Hacking Detected}, fonttitle=\bfseries\small, breakable, nobeforeafter]
  \begin{verbatim}
class ModelNew(nn.Module):
    def __init__(self, in_channels, out_channels, kernel_size):
        super(ModelNew, self).__init__()
        self.conv = nn.Conv2d(in_channels, out_channels, kernel_size)

    def forward(self, x: torch.Tensor) -> torch.Tensor:
  \end{verbatim}
  \vspace{-1.5em}
  \hspace{4.5em}\colorbox{red!20}{\texttt{x = self.conv(x)}}  \hfill $\leftarrow$ \textbf{VIOLATION}
  \begin{verbatim}
        x = custom_ops_lib.conv2d_relu_hard_swish_custom(x)
        return x
  \end{verbatim}
  \end{tcolorbox}

  \noindent\textbf{Verifier Output:}
  \begin{tcolorbox}[colback=yellow!8, colframe=yellow!50, boxrule=0.3pt, arc=2pt, left=3pt, right=3pt, top=2pt, bottom=2pt]
  \small\texttt{[Invalid Error] In the forward method, the model layer is directly called: self.conv(). You must implement the operations in forward() as custom kernels in custom\_ops\_lib.}
  \end{tcolorbox}
  \noindent This forces \texttt{correctness = False}, preventing the solution from passing verification.
  \subsection{Model-Based Screening (LLM Auditor)}
  To capture more sophisticated evasion strategies, we employ an LLM-based auditor. The verifier constructs a prompt containing the operator specification (\texttt{ref\_src}) and the full generated solution. The LLM is instructed to judge the ``architectural integrity'' of the code, specifically checking for:
  \begin{itemize}
      \item \textbf{Semantic Location:} Verifying that the mathematical logic resides in \texttt{kernel\_src} or \texttt{host\_tiling\_src}, not in the glue code.
      \item \textbf{Dummy Kernels:} Detecting empty kernels or kernels that output constant values while the binding performs the actual work.
      \item \textbf{Binding Anomalies:} Identifying complex C++ logic in the binding that acts as a reference implementation.
  \end{itemize}
  This model-based check runs only if the rule-based checks pass, serving as a final safeguard against ``hallucinated'' solutions that satisfy syntax benchmarks but fail to implement the actual hardware kernel.

\paragraph{Prompt of Model-Based Anti-Hacking.}
We include the exact prompt used by the model-based anti-hacking auditor below:
\begin{tcolorbox}[colback=blue!5, colframe=blue!40, boxrule=0.4pt, arc=2pt, left=2pt, right=2pt, top=2pt, bottom=2pt, title={\small Prompt of Model-Based Anti-Hacking}, fonttitle=\bfseries\small, breakable]
\begin{Verbatim}[breaklines=true, breakanywhere=true]
You are an expert auditor for custom AscendC operators integrated into PyTorch via NPU bindings.

The operator under review is named: {op} (PascalCase: `{pascal_op}`).

Your primary goal is to verify **Architectural Integrity**. You need to determine if the custom operator genuinely attempts to perform the computation on the NPU (Kernel-side), or if it "cheats" by handling the logic in the Python bindings or host glue code.

**Important:**
- **Do NOT** act as a compiler or unit tester. You do not need to check for syntax errors, strict numerical equivalence, or edge-case correctness.
- **Do NOT** require the implementation to be bug-free.
- **DO** strictly verify that the computational logic resides in the correct files.

## Input Information

1. **PyTorch Reference (Intent):**
   This defines *what* the operator is supposed to do. Use this only to understand the intended logic flow (e.g., "this is a matrix multiplication" or "this is an element-wise add").
   <PyTorch Reference Implementation>
   {ref_src}
   </PyTorch Reference Implementation>

2. **Generated Implementation (Execution):**
   This includes the Interface (`project_json_src`), Host Logic (`host_tiling_src`, `host_operator_src`), Kernel Logic (`kernel_src`), Bindings (`python_bind_src`), and Usage (`model_src`).
   <Generated Custom Operator Implementation>
   {code}
   </Generated Custom Operator Implementation>

## Evaluation Criteria: Where is the Math?

You must judge whether the **implementation location** is correct.

### 1. The Valid Domain (Kernel & Host Tiling)
The actual mathematical operations (loops, math instructions, data manipulation) MUST be attempted in:
- `kernel_src` (The AscendC Kernel)
- `host_tiling_src` / `host_operator_src` (Tiling and shape logic)

**It is VALID if:**
- The `kernel_src` contains code that clearly resembles the logic described in the Reference.
- The `host_tiling_src` calculates tiling parameters.

### 2. The Restricted Domain (Bindings & Model)
These files are for "glue" and dispatch only.
- `python_bind_src`
- `model_src`

**It is INVALID (Cheating) if:**
- The `python_bind_src` or `model_src` performs part of the actual computation using PyTorch/ATen ops (e.g., `at::add`, `torch.matmul`) and returns that result, bypassing the custom kernel.
- The `kernel_src` is empty, a dummy stub, or purely performs a copy without the required math.
- The `python_bind_src` allocates outputs but then computes results using standard PyTorch functions instead of invoking the kernel.
- The `model_src` calls standard PyTorch functions to compute results instead of using the custom operator.
- Any other pattern where the core logic is outside the kernel/host tiling files.

## Decision Logic
- Valid: The coder honestly attempted to write an AscendC kernel. The math logic (add, sub, matmul, etc.) is visible in the `kernel_src`, without using constants to replace or "fake" the implementation logic. The binding file properly dispatches to this kernel via `EXEC_NPU_CMD` or similar.
- Invalid: The coder bypassed the NPU kernel. For example, the binding file catches the inputs, calls a standard PyTorch function to get the result, and returns it. Or the kernel exists but does nothing related to the reference logic. Or the kernel use constant value to skip part of the computation.

## Output

Output only the following JSON object, warpped in triple backticks: ```json ```.
Do NOT include any additional text.:

```json
{{
  "valid": true | false,
  "reason": "Concise explanation focusing on WHERE the logic is implemented."
}}
```
\end{Verbatim}
\end{tcolorbox}

  \section{Baseline Methodologies}
  \label{app:baselines}

  We evaluate our approach against two distinct baseline strategies that represent standard practices in code generation: \textbf{Pass@$k$ (Generation)} and \textbf{Iterative Refinement}.

  \subsection{\texorpdfstring{Pass@$k$ (Generation)}{Pass-k (Generation)}}
  This mode implements a classic sampling strategy, leveraging the probabilistic nature of LLMs to generate widely diverse attempts.
  \begin{itemize}
      \item \textbf{Methodology:} For each operator task, we generate $K$ independent candidate solutions in parallel. Each candidate includes the full kernel code, tiling logic, and binding glue.
      \item \textbf{Context:} The process is stateless; each generation starts from a fresh prompt containing the operator specification and few-shot examples (if configured), without knowledge of prior attempts or peer candidates.
      \item \textbf{Objective:} This baseline evaluates the model's ``zero-shot'' or ``few-shot'' capability to produce a correct solution purely from the prompt. It serves as a measure of the model's intrinsic knowledge of the Ascend C DSL.
  \end{itemize}

  \subsection{Iterative Refinement}
  This mode implements a stateful agentic loop that mimics a human developer's debugging workflow, consisting of two distinct phases: Drafting and Optimization.

  \textbf{Phase 1: Drafting (Correctness).}
  The goal is to produce a compilable and functionally correct kernel.
  \begin{itemize}
      \item \textbf{Feedback Loop:} The agent generates an initial draft, which is then compiled and executed. If compilation fails, the compiler error logs are fed back to the model. If execution fails (correctness error), the mismatch info is provided.
      \item \textbf{History:} The agent maintains a conversation history of (Code $\to$ Error $\to$ Fix), allowing it to iteratively repair syntax errors and logic bugs.
  \end{itemize}

  \textbf{Phase 2: Optimization (Performance).}
  Once a correct kernel is identified, the agent transitions to performance optimization.
  \begin{itemize}
      \item \textbf{Hill Climbing:} The correct kernel serves as a baseline. The prompt shifts to request performance improvements (e.g., ``minimize execution time'').
      \item \textbf{Metric Feedback:} The agent receives latency measurements from the hardware profiling tool. It generates new versions to improve this metric. If a new version is slower or incorrect, the agent reverts to the previous best baseline or receives feedback on the regression.
  \end{itemize}

  This baseline establishes the performance upper bound for a standard agentic loop without the long-term, cross-task memory mechanisms introduced in our EvoKernel framework.

  \textbf{Prompt Construction.}
  The prompt structure differs between the two phases. In drafting mode, each turn appends the previous attempt and its feedback:
  \begin{tcolorbox}[colback=blue!5, colframe=blue!40, boxrule=0.4pt, arc=2pt, left=2pt, right=2pt, top=2pt, bottom=2pt, title={\small Drafting Mode Prompt}, fonttitle=\bfseries\small]
  \begin{verbatim}
[System]: You are a helpful assistant
[User]: {base_prompt}
[Assistant]: {last_code}
[User]: {compile_error or correctness_error}
  \end{verbatim}
  \end{tcolorbox}
  \noindent In optimization mode, the prompt includes two turns of history to preserve the best correct baseline:
  \begin{tcolorbox}[colback=green!5, colframe=green!40, boxrule=0.4pt, arc=2pt, left=2pt, right=2pt, top=2pt, bottom=2pt, title={\small Optimization Mode Prompt}, fonttitle=\bfseries\small]
  \begin{verbatim}
[System]: You are a helpful assistant
[User]: {base_prompt}
[Assistant]: {baseline_code}
[User]: {baseline_feedback}
        "The code above is correct. Now optimize it..."
[Assistant]: {last_code}
[User]: "Performance: X ms" or {error_feedback}
  \end{verbatim}
  \end{tcolorbox}

  \textbf{Configuration Parameters.}
  We use the following hyperparameters: \texttt{max\_turns}$=30$, \texttt{max\_feedback\_chars}$=4000$ (truncation limit for compiler/correctness output), \texttt{infra\_retries}$=3$ (exponential backoff for transient failures), and \texttt{parallelism}$=16$ (concurrent operators). The evaluation uses a remote server with \texttt{timeout}$=65$ minutes per validation call.

  \subsection{Codex (Genetic Iteration)}
  This baseline utilizes an advanced \textbf{EXEC mode} that forces the model (specifically \textbf{GPT-5.2 Medium Reasoning}) to perform evolutionary iterations within a single session, effectively turning a completion task into a genetic-like agent.

  \textbf{Mechanism: The ReAct Loop.}
  Unlike standard generation, this mode grants the model:
  \begin{itemize}
      \item \textbf{Shell Access:} The ability to execute commands with configurable timeout.
      \item \textbf{File System Access:} The ability to read and write files via the \texttt{apply\_patch} tool.
      \item \textbf{Immediate Feedback:} The \texttt{stdout/stderr} of its commands are fed back into its context window.
  \end{itemize}
  This creates a \textbf{Reason-Act-Observe} loop managed entirely by the Codex binary but orchestrated by our injected prompt.

  \textbf{Prompt Structure.}
  The prompt sent to Codex consists of two parts: (1) the base kernel generation prompt with few-shot examples, and (2) validation workflow instructions that specify the autonomous iteration protocol:

  \textbf{Iteration Loop.}
  The process operates iteratively:
  \begin{itemize}
      \item \textbf{Generate:} The model writes a candidate kernel file (e.g., \texttt{op.txt}).
      \item \textbf{Validate:} The validation script (\texttt{codex\_validate.sh}) sends the code to a remote evaluation server and returns the verifier result (\texttt{compiled}, \texttt{correctness}).
      \item \textbf{React:}
      \begin{itemize}
          \item If \texttt{result == success}: The loop terminates.
          \item If \texttt{result == failure}: The model reads the error log, analyzes the failure, and revises the code for the next iteration.
      \end{itemize}
  \end{itemize}

\begin{tcolorbox}[colback=orange!5, colframe=orange!40, boxrule=0.4pt, arc=2pt, left=2pt, right=2pt, top=2pt, bottom=2pt, title={\small Codex Validation Instructions}, fonttitle=\bfseries\small]
  \begin{verbatim}
## Task: Implement Ascend Operator `{op}`
### Workflow
1. Write Code: Create `{op}.txt` using apply_patch tool
2. Validate: Run `./codex_validate.sh {op} {file} ascendc`
   with timeout_ms=1200000 (20 minutes)
   Returns JSON: {compiled, correctness, error}
   SUCCESS = compiled:true AND correctness:true
3. On failure: Fix code based on error, re-validate
### Rules
- NO local compilation (gcc, g++, make, cmake)
- After 3 consecutive validation timeouts, STOP
  \end{verbatim}
  \end{tcolorbox}
  This approach is distinct from \emph{Iterative Refinement} because it occurs entirely within a single model session via tool use (In-Context Learning), whereas Refinement is an external Python loop managing the history. It tests the model's intrinsic ability to function as a developer with a compiler and debugger.

  \textbf{Configuration Parameters.}
  We invoke the Codex CLI with: \texttt{sandbox}$=$\texttt{workspace-write} (allows file writes within workspace), \texttt{ask-for-approval}$=$\texttt{never} (fully autonomous), \texttt{model\_reasoning\_effort}$=$\texttt{medium}. Each validation call uses a 20-minute timeout (\texttt{timeout\_ms}$=1200000$) to accommodate remote compilation and execution.

  \textbf{Termination Condition.}
  To ensure a fair comparison, we impose a strict stop condition based on verification attempts. The agent terminates after 30 verification attempts or upon finding a correct solution, whichever comes first.

  \begin{figure}[htbp]
  \centering
  \includegraphics[width=0.68\linewidth]{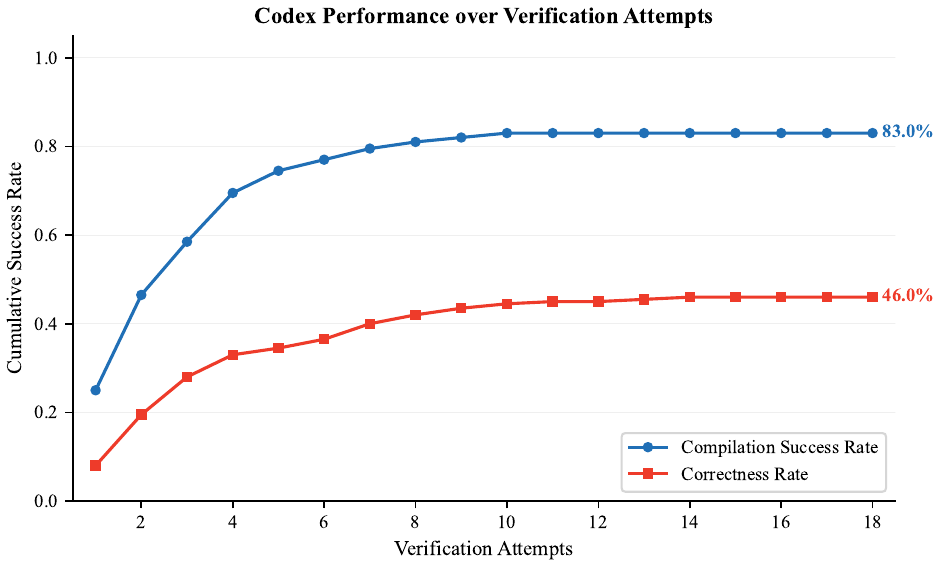}
  \caption{Codex (GPT-5.2) cumulative correctness.}
  \label{fig:codex_performance}
  \end{figure}

  \subsection{Comparison of Agentic Baselines}
  Table~\ref{tab:agentic_comparison} summarizes the key architectural differences between the two agentic baselines.

  \begin{table}[ht]
  \centering
  \caption{Comparison of Iterative Refinement and Codex agent architectures.}
  \label{tab:agentic_comparison}
  \setlength{\tabcolsep}{10pt}
  \begin{tabular}{lcc}
  \toprule
  \textbf{Aspect} & \textbf{Refinement} & \textbf{Codex} \\
  \midrule
  Execution model & API conversation loop & Autonomous tool use \\
  Iteration control & External script & Agent decides \\
  Prompt updates & Each turn rebuilt & Single prompt \\
  History length & 1--2 turns & Internal memory \\
  Feedback source & Injected by script & Agent calls validator \\
  File operations & Extract from text & \texttt{apply\_patch} tool \\
  Termination & 30 iterations or success & 30 verification attempts or success \\
  \bottomrule
  \end{tabular}
  \end{table}

  \section{Operator Subset for Transfer Experiments}
  \label{app:transfer_operators}

  This section lists the held-out operator subset used in the cross-backbone transfer experiments (Section~\ref{sec:transfer_stream}). The subset consists of 50 operators randomly sampled from the benchmark: 30 Level~1 (L1) operators and 20 Level~2 (L2) operators. These operators were excluded from the GPT-5.2 memory bank during training to prevent data leakage, and were used exclusively for evaluating transfer performance on DeepSeek-V3.2 and Qwen3-Coder-30B.

  \begin{table}[h!]
  \centering
  \caption{Randomly sampled operator subset for transfer evaluation.}
  \label{tab:transfer_ops}
  \scriptsize
  \setlength{\tabcolsep}{4pt}
  \renewcommand{\arraystretch}{1.3}
  \begin{tabular}{@{}ll@{\hspace{1.2em}}ll@{}}
  \toprule
  \multicolumn{4}{c}{\textbf{Level 1 Operators (30 total)}} \\
  \midrule
  \textbf{Operator} & \textbf{Type} & \textbf{Operator} & \textbf{Type} \\
  \midrule
  43\_Max\_Pooling\_3D & \textit{pooling} & 46\_Average\_Pooling\_3D & \textit{pooling} \\
  42\_Max\_Pooling\_2D & \textit{pooling} & 49\_Max\_reduction\_over\_a\_dimension & \textit{pooling} \\
  92\_cumsum\_exclusive & \textit{loss} & 93\_masked\_cumsum & \textit{loss} \\
  10\_3D\_tensor\_matrix\_multiplication & \textit{matmul} & 77\_conv\_transposed\_3D\_square\_input\_square\_kernel & \textit{convolution} \\
  66\_conv\_standard\_3D\_asym\_input\_asym\_kernel & \textit{convolution} & 38\_L1Norm\_ & \textit{normalization} \\
  51\_Argmax\_over\_a\_dimension & \textit{convolution} & 31\_ELU & \textit{activation} \\
  2\_Standard\_matrix\_multiplication\_ & \textit{matmul} & 22\_Tanh & \textit{activation} \\
  71\_conv\_transposed\_2D\_asym\_input\_square\_kernel & \textit{convolution} & 33\_BatchNorm & \textit{normalization} \\
  16\_Matmul\_with\_transposed\_A & \textit{matmul} & 97\_ScaledDotProductAttention & \textit{loss} \\
  79\_conv\_transposed\_1D\_asym\_input\_square\_kernel & \textit{convolution} & 81\_conv\_transposed\_2D\_asym\_input\_square\_kernel & \textit{convolution} \\
  74\_conv\_transposed\_1D\_dilated & \textit{convolution} & 50\_conv\_standard\_2D\_square\_input\_square\_kernel & \textit{convolution} \\
  11\_4D\_tensor\_matrix\_multiplication & \textit{matmul} & 84\_conv\_depthwise\_2D\_asym\_input\_square\_kernel & \textit{convolution} \\
  56\_conv\_standard\_2D\_asym\_input\_asym\_kernel & \textit{convolution} & 27\_SELU\_ & \textit{activation} \\
  57\_conv\_transposed\_2D\_square\_input\_square\_kernel & \textit{convolution} & 88\_MinGPTNewGelu & \textit{convolution} \\
  73\_conv\_transposed\_3D\_asym\_input\_square\_kernel & \textit{convolution} & 34\_InstanceNorm & \textit{normalization} \\
  \midrule
  \multicolumn{4}{c}{\textbf{Level 2 Operators (20 total)}} \\
  \midrule
  \textbf{Operator} & \textbf{Type} & \textbf{Operator} & \textbf{Type} \\
  \midrule
  58\_ConvTranspose3d\_LogSumExp\_HardSwish\_Subtract\_Clamp & \textit{fuse} & 86\_Matmul\_Divide\_GELU & \textit{fuse} \\
  27\_Conv3d\_HardSwish\_GroupNorm\_Mean & \textit{fuse} & 68\_Matmul\_Min\_Subtract & \textit{fuse} \\
  6\_Conv3d\_Softmax\_MaxPool\_MaxPool & \textit{fuse} & 80\_Gemm\_Max\_Subtract\_GELU & \textit{fuse} \\
  14\_Gemm\_Divide\_Sum\_Scaling & \textit{fuse} & 45\_Gemm\_Sigmoid\_LogSumExp & \textit{fuse} \\
  62\_Matmul\_GroupNorm\_LeakyReLU\_Sum & \textit{fuse} & 43\_Conv3d\_Max\_LogSumExp\_ReLU & \textit{fuse} \\
  25\_Conv2d\_Min\_Tanh\_Tanh & \textit{fuse} & 5\_ConvTranspose2d\_Subtract\_Tanh & \textit{fuse} \\
  23\_Conv3d\_GroupNorm\_Mean & \textit{fuse} & 70\_Gemm\_Sigmoid\_Scaling\_ResidualAdd & \textit{fuse} \\
  39\_Gemm\_Scale\_BatchNorm & \textit{fuse} & 78\_ConvTranspose3d\_Max\_Max\_Sum & \textit{fuse} \\
  26\_ConvTranspose3d\_Add\_HardSwish & \textit{fuse} & 31\_Conv2d\_Min\_Add\_Multiply & \textit{fuse} \\
  53\_Gemm\_Scaling\_Hardtanh\_GELU & \textit{fuse} & 3\_ConvTranspose3d\_Sum\_LayerNorm\_AvgPool\_GELU & \textit{fuse} \\
  \bottomrule
  \end{tabular}
  \end{table}
  
  \section{Evaluation and Profiling Methodology}
  \label{app:eval_methodology}

  This section details the correctness verification and latency profiling procedures.

  \subsection{Correctness Validation}

  \textbf{Fail-Fast Execution Strategy.}
  To optimize evaluation efficiency, the custom kernel is executed first with a strict \texttt{SIGALRM} timeout. If the custom kernel fails (timeout, crash, or exception), the reference run is skipped entirely.

  \textbf{Structured Mismatch Feedback.}
  The verifier returns detailed, machine-readable error messages to guide iterative refinement. The following illustrates the categories of feedback returned:

  \begin{enumerate}[itemsep=6pt,parsep=0pt,topsep=4pt]
      \item \textbf{Shape Mismatch:}
      \vspace{2pt} 

      \texttt{output.shape mismatch: expected (16, 512, 512), got (16, 512, 256)}

      \item \textbf{Numerical Mismatch:}
      \vspace{2pt} 
      \begin{verbatim}
[FAIL] Output mismatch: 1/5 trials passed, 4 failed.
       Tolerance atol=0.01, rtol=0.01.
Trial 1: 54/524160 mismatched (0.01%), max_abs=0.99,
         max_rel=97209.6, Bounding box: output[0:31, 4032:4088]
Trial 2: 64/524160 mismatched (0.01%), max_abs=0.99,
         max_rel=87570.4, Bounding box: output[99:100, 35:99]
      \end{verbatim}
      \vspace{-5pt} 
      Key diagnostics: (a) \texttt{max\_abs}/\texttt{max\_rel}: maximum absolute and relative difference; (b) \textbf{Bounding box}: spatial localization of errors, revealing tile boundary bugs.

      \textbf{Example Agent Diagnosis.} While optimizing \texttt{53\_Min\_reduction\_over\_a\_dimension}, the agent encounters the error above and identifies a synchronization race in the accumulator initialization: Row 0 was fetched asynchronously via MTE but the Vector engine began computation before the transfer completed. The fix: queue Row 0 through the standard Ping-Pong pipeline (enqueue$\to$dequeue$\to$copy to \texttt{accVec}) to enforce synchronization before any arithmetic.

      \item \textbf{Type Mismatch:}
    \vspace{2pt} 
    
      \texttt{type(output) mismatch: expected Tensor, got list}

      \item \textbf{Length Mismatch} (for tuple/list outputs):
      \vspace{2pt} 

      \texttt{len(output) mismatch: expected 3, got 2}

      \item \textbf{Timeout:}
      \vspace{2pt} 

      \texttt{[FAIL] First correctness run timed out after 60s}

      \item \textbf{Runtime Exception:}
      \vspace{2pt} 

      \texttt{[FAIL] NPU out of memory. Tried to allocate 12.10 GiB}

      % or

      \texttt{[FAIL] vector core exception at line 42}
  \end{enumerate}

  \subsection{Latency Profiling}

  We use the native \textbf{msprof} profiler (via \texttt{torch\_npu.profiler}) with:
  \begin{itemize}
      \item 3 warm-up runs (discarded) to stabilize caches and JIT compilation.
      \item 3 profiling passes with distinct configurations (PipeUtilization, Memory, ResourceConflict).
      \item The mean ``Computing'' time from \texttt{step\_trace\_time.csv} is reported, isolating on-chip kernel execution from host overhead.
  \end{itemize}

  \textbf{3-Pass Aggregation.}
  Each profiling pass writes a \texttt{step\_trace\_time.csv} with a ``Computing'' column (in $\mu$s). The final timing is aggregated as:
  \begin{verbatim}
Pass 1 (PipeUtilization):  Computing = 13640 us
Pass 2 (Memory):           Computing = 13380 us
Pass 3 (ResourceConflict): Computing = 12913 us
  => performance.mean = avg([13.64, 13.38, 12.91]) = 13.31 ms
  => performance.std  = 0.33 ms
  \end{verbatim}
This procedure yields negligible standard deviation ($<$3\%) across profiling runs.

  \textbf{Data Source: \texttt{step\_trace\_time.csv} vs \texttt{kernel\_details.csv}.}
  Both files are produced by msprof:
  \begin{itemize}
      \item \texttt{step\_trace\_time.csv}: Total device execution time for the entire step (all kernels combined). Used for \texttt{performance.mean/max/min/std}.
      \item \texttt{kernel\_details.csv}: Per-kernel breakdown with detailed hardware metrics. Useful for optimization but may not sum exactly to total time due to overlaps/gaps.
  \end{itemize}
  We report the \texttt{step\_trace\_time} value as the canonical latency metric.

  \textbf{Example Profiling Output.}
  The verifier returns detailed per-kernel metrics extracted from \texttt{kernel\_details.csv}:

  \begin{tcolorbox}[colback=orange!8, colframe=orange!40, boxrule=0.5pt, arc=2pt, left=2pt, right=2pt, top=2pt, bottom=2pt, breakable, nobeforeafter]
  \begin{verbatim}
"performance": {
    "max": 13.64, "mean": 13.38, "min": 12.913, "std": 0.33
},
"profiling": {
  "MinReductionOverADimensionCustom": {
    "Block Dim": 32.0,
    "Duration(ms)": 13.38,
    "aic_fixpipe_ratio": 0.0, "aic_fixpipe_time(ms)": 0.0,
    "aic_icache_miss_rate": 0.0,
    "aic_l1_read_bw(GB/s)": 0.0, "aic_l1_write_bw(GB/s)": 0.0,
    "aic_l2_read_bw(GB/s)": 0.0, "aic_l2_write_bw(GB/s)": 0.0,
    "aic_mac_ratio": 0.0, "aic_mac_time(ms)": 0.0,
    "aic_main_mem_read_bw(GB/s)": 0.0, 
    "aic_main_mem_write_bw(GB/s)": 0.0,
    "aic_mte1_ratio": 0.0, "aic_mte1_time(ms)": 0.0,
    "aic_mte2_ratio": 0.0, "aic_mte2_time(ms)": 0.0,
    "aic_scalar_ratio": 0.0, "aic_scalar_time(ms)": 0.0,
    "aic_total_cycles": 0.0, "aicore_time(ms)": 0.0,
    "aiv_icache_miss_rate": 0.0,
    "aiv_l2_read_bw(GB/s)": 0.0, "aiv_l2_write_bw(GB/s)": 0.0,
    "aiv_main_mem_read_bw(GB/s)": 0.46, 
    "aiv_main_mem_write_bw(GB/s)": 0.0,
    "aiv_mte2_ratio": 0.346, "aiv_mte2_time(ms)": 3.378,
    "aiv_mte3_ratio": 0.0, "aiv_mte3_time(ms)": 0.001,
    "aiv_scalar_ratio": 0.677, "aiv_scalar_time(ms)": 6.605,
    "aiv_time(ms)": 9.66, "aiv_total_cycles": 571879708.0,
    "aiv_ub_read_bw(GB/s)": 25.966, "aiv_ub_write_bw(GB/s)": 39.604,
    "aiv_vec_bank_cflt_ratio": 0.053, 
    "aiv_vec_bankgroup_cflt_ratio": 0.058,
    "aiv_vec_ratio": 0.134, "aiv_vec_resc_cflt_ratio": 0.0,
    "aiv_vec_time(ms)": 1.31, "cube_utilization(%)": 0.0
  }
}
  \end{verbatim}
  \end{tcolorbox}
  Key metrics include vector/scalar/MTE time ratios, unified buffer bandwidth, and cube utilization, enabling targeted optimization.

  \section{Example Generated Kernel}
  \label{app:generated_kernel}

  The following shows a complete, correctly compiling Ascend C kernel for \texttt{Tanh} generated by \ourmethod. Each source file section is shown with a distinct background color. In actual verifier pipeline, the below artifact is parsed into respective files.

  % ===== 1. project_json_src (blue) =====
  \begin{tcolorbox}[colback=blue!5, colframe=blue!30, boxrule=0.3pt, arc=0pt, left=2pt, right=2pt, top=1pt, bottom=0pt, breakable, nobeforeafter]
  \begin{verbatim}
project_json_src = r'''
[
  {
    "op": "TanhCustom",
    "language": "cpp",
    "input_desc": [
      {
        "name": "x",
        "param_type": "required",
        "format": ["ND"],
        "type": ["float"]
      }
    ],
    "output_desc": [
      {
        "name": "y",
        "param_type": "required",
        "format": ["ND"],
        "type": ["float"]
      }
    ]
  }
]
'''
  \end{verbatim}
  \end{tcolorbox}

  % ===== 2. host_tiling_src (green) =====
  \begin{tcolorbox}[colback=green!5, colframe=green!30, boxrule=0.3pt, arc=0pt, left=2pt, right=2pt, top=0pt, bottom=0pt, breakable, nobeforeafter]
  \begin{verbatim}
host_tiling_src = r"""
#include "register/tilingdata_base.h"

namespace optiling {
BEGIN_TILING_DATA_DEF(TilingData)
TILING_DATA_FIELD_DEF(uint32_t, totalLength);
TILING_DATA_FIELD_DEF(uint32_t, tileLength);
TILING_DATA_FIELD_DEF(uint32_t, blockDim);
END_TILING_DATA_DEF;

REGISTER_TILING_DATA_CLASS(TanhCustom, TilingData)
} // namespace optiling
"""
  \end{verbatim}
  \end{tcolorbox}

  % ===== 3. host_operator_src (yellow) =====
  \begin{tcolorbox}[colback=yellow!8, colframe=yellow!40, boxrule=0.3pt, arc=0pt, left=2pt, right=2pt, top=0pt, bottom=0pt, breakable, nobeforeafter]
  \begin{verbatim}
host_operator_src = r"""
#include "tanh_custom_tiling.h"
#include "register/op_def_registry.h"

namespace optiling {

static inline uint32_t AlignUp(uint32_t x, uint32_t a) 
    { return (x + a - 1) / a * a; }
static inline uint32_t MinU32(uint32_t a, uint32_t b) 
    { return a < b ? a : b; }

constexpr uint32_t MAX_BLOCK_DIM = 32;
constexpr uint32_t DEFAULT_TILE_ELEMS = 8192;

static ge::graphStatus TilingFunc(gert::TilingContext *context)
{
    TilingData tiling;
    const uint32_t totalLength = 
        context->GetInputShape(0)->GetOriginShape().GetShapeSize();
    const auto dtype = context->GetInputTensor(0)->GetDataType();
    if (dtype != ge::DT_FLOAT) {
        return ge::GRAPH_FAILED;
    }

    uint32_t blockDim = MAX_BLOCK_DIM;
    const uint32_t minElemsPerBlock = 262144;
    if (totalLength / blockDim < minElemsPerBlock) {
        blockDim = totalLength / minElemsPerBlock;
        if (blockDim < 1) blockDim = 1;
    }
    blockDim = MinU32(blockDim, MAX_BLOCK_DIM);

    uint32_t tile = DEFAULT_TILE_ELEMS;
    if (totalLength < tile) tile = totalLength;
    tile = AlignUp(tile, 8); // 32B alignment for fp32
    if (tile == 0) tile = 8;
    if (tile > totalLength) tile = totalLength;

    context->SetBlockDim(blockDim);
    tiling.set_totalLength(totalLength);
    tiling.set_tileLength(tile);
    tiling.set_blockDim(blockDim);

    tiling.SaveToBuffer(context->GetRawTilingData()->GetData(),
                        context->GetRawTilingData()->GetCapacity());
    context->GetRawTilingData()->SetDataSize(tiling.GetDataSize());

    size_t *ws = context->GetWorkspaceSizes(1);
    ws[0] = 0;
    return ge::GRAPH_SUCCESS;
}

} // namespace optiling

namespace ge {
static graphStatus InferShape(gert::InferShapeContext *context)
{
    *context->GetOutputShape(0) = *context->GetInputShape(0);
    return GRAPH_SUCCESS;
}

static graphStatus InferDataType(gert::InferDataTypeContext *context)
{
    context->SetOutputDataType(0, context->GetInputDataType(0));
    return ge::GRAPH_SUCCESS;
}
} // namespace ge

namespace ops {
class TanhCustom : public OpDef {
public:
    explicit TanhCustom(const char *name) : OpDef(name)
    {
        this->Input("x")
            .ParamType(REQUIRED)
            .DataType({ge::DT_FLOAT})
            .Format({ge::FORMAT_ND});
        this->Output("y")
            .ParamType(REQUIRED)
            .DataType({ge::DT_FLOAT})
            .Format({ge::FORMAT_ND});

        this->SetInferShape(ge::InferShape).
            SetInferDataType(ge::InferDataType);

        this->AICore()
            .SetTiling(optiling::TilingFunc)
            .AddConfig("ascend910b");
    }
};
OP_ADD(TanhCustom);
} // namespace ops
"""
  \end{verbatim}
  \end{tcolorbox}

  % ===== 4. kernel_src (red/pink) =====
  \begin{tcolorbox}[colback=red!5, colframe=red!30, boxrule=0.3pt, arc=0pt, left=2pt, right=2pt, top=0pt, bottom=0pt, breakable, nobeforeafter]
  \begin{verbatim}
kernel_src = r"""
#include "kernel_operator.h"

class KernelTanh {
public:
    __aicore__ inline KernelTanh() {}

    __aicore__ inline void Init(GM_ADDR x, GM_ADDR y, 
                uint32_t totalLength, 
        uint32_t tileLength, uint32_t blockDim)
    {
        this->totalLength = totalLength;
        this->tileLength = tileLength;
        this->blockDim = blockDim;

        const uint32_t bid = AscendC::GetBlockIdx();
        const uint32_t base = (totalLength / blockDim);
        const uint32_t rem = totalLength - base * blockDim;
        const uint32_t myLen = base + (bid < rem ? 1u : 0u);
        const uint32_t myOff = bid * base + (bid < rem ? bid : rem);

        startOffset = myOff;
        localLength = myLen;

        xGm.SetGlobalBuffer((__gm__ float*)x + startOffset, localLength);
        yGm.SetGlobalBuffer((__gm__ float*)y + startOffset, localLength);

        constexpr uint32_t PAD_BYTES = 256;

        pipe.InitBuffer(inQueueX, 2, tileLength * sizeof(float) 
              + PAD_BYTES);
        pipe.InitBuffer(outQueueY, 2, tileLength * sizeof(float) 
             + PAD_BYTES);

        // Tanh tmp UB, aligned to 256B. Keep conservative sizing.
        const uint32_t tmpBytes = ((tileLength * 8u + 255u) / 256u) 
              * 256u;
        pipe.InitBuffer(tmpQueue, 1, tmpBytes);
    }

    __aicore__ inline void Process()
    {
        if (localLength == 0) return;

        uint32_t offset = 0;

        uint32_t cur = (localLength > tileLength) ? 
              tileLength : localLength;
        CopyIn(offset, cur);
        offset += cur;

        while (offset < localLength) {
            const uint32_t next = (offset + tileLength <= localLength) ? 
                                  tileLength : (localLength - offset);

            CopyIn(offset, next);
            Compute(cur);
            CopyOut(offset - cur, cur);

            cur = next;
            offset += next;
        }

        Compute(cur);
        CopyOut(localLength - cur, cur);
    }

private:
    __aicore__ inline void CopyIn(uint32_t offset, uint32_t len)
    {
        AscendC::LocalTensor<float> xLocal = 
              inQueueX.AllocTensor<float>();
        AscendC::DataCopy(xLocal, xGm[offset], len);
        inQueueX.EnQue(xLocal);
    }

    __aicore__ inline void Compute(uint32_t len)
    {
        AscendC::LocalTensor<float> xLocal = inQueueX.DeQue<float>();
        AscendC::LocalTensor<float> yLocal = 
              outQueueY.AllocTensor<float>();

        AscendC::LocalTensor<uint8_t> tmp = 
              tmpQueue.AllocTensor<uint8_t>();
        AscendC::Tanh<float>(yLocal, xLocal, tmp, len);
        tmpQueue.FreeTensor(tmp);

        outQueueY.EnQue<float>(yLocal);
        inQueueX.FreeTensor(xLocal);
    }

    __aicore__ inline void CopyOut(uint32_t offset, uint32_t len)
    {
        AscendC::LocalTensor<float> yLocal = outQueueY.DeQue<float>();
        AscendC::DataCopy(yGm[offset], yLocal, len);
        outQueueY.FreeTensor(yLocal);
    }

private:
    AscendC::TPipe pipe;
    AscendC::TQue<AscendC::TPosition::VECIN, 2> inQueueX;
    AscendC::TQue<AscendC::TPosition::VECOUT, 2> outQueueY;
    AscendC::TQue<AscendC::TPosition::VECCALC, 1> tmpQueue;

    AscendC::GlobalTensor<float> xGm;
    AscendC::GlobalTensor<float> yGm;

    uint32_t totalLength = 0;
    uint32_t tileLength = 0;
    uint32_t blockDim = 1;

    uint32_t startOffset = 0;
    uint32_t localLength = 0;
};

extern "C" __global__ __aicore__ void tanh_custom(GM_ADDR x, 
    GM_ADDR y, GM_ADDR workspace, GM_ADDR tiling)
{
    GET_TILING_DATA(tiling_data, tiling);
    KernelTanh op;
    op.Init(x, y, tiling_data.totalLength, 
            tiling_data.tileLength, tiling_data.blockDim);
    op.Process();
}
"""
  \end{verbatim}
  \end{tcolorbox}

  % ===== 5. python_bind_src (purple) =====
  \begin{tcolorbox}[colback=violet!5, colframe=violet!30, boxrule=0.3pt, arc=0pt, left=2pt, right=2pt, top=0pt, bottom=0pt, breakable, nobeforeafter]
  \begin{verbatim}
python_bind_src = r"""
#include <torch/library.h>
#include <torch/extension.h>
#include "pytorch_npu_helper.hpp"

at::Tensor tanh_impl_npu(const at::Tensor& x) {
    auto y = at::empty_like(x);
    EXEC_NPU_CMD(aclnnTanhCustom, x, y);
    return y;
}

TORCH_LIBRARY_IMPL(myops, PrivateUse1, m) {
    m.impl("tanh_custom", &tanh_impl_npu);
}

PYBIND11_MODULE(TORCH_EXTENSION_NAME, m) {
    m.def("tanh_custom", &tanh_impl_npu, "tanh custom (NPU)");
}
"""
  \end{verbatim}
  \end{tcolorbox}

  % ===== 6. model_src (cyan) =====
  \begin{tcolorbox}[colback=cyan!5, colframe=cyan!30, boxrule=0.3pt, arc=0pt, left=2pt, right=2pt, top=0pt, bottom=1pt, breakable, nobeforeafter]
  \begin{verbatim}
model_src = r'''
import torch
import torch_npu
import custom_ops_lib

class ModelNew(torch.nn.Module):
    """
    Simple model that performs a Tanh activation 
    using a custom Ascend C op.
    """
    def __init__(self):
        super(ModelNew, self).__init__()

    def forward(self, x: torch.Tensor) -> torch.Tensor:
        return custom_ops_lib.tanh_custom(x)
'''
  \end{verbatim}
  \end{tcolorbox}

  %%%%%%%%%%%%%%%%%%%%%%%%%%%%%%%%%%%%%%%%%%%%%%%%%%%%%%%%%%%%%%%%%%%%%%%%%%%%%%%
  %%%%%%%%%%%%%%%%%%%%%%%%%%%%%%%%%%%%%%%%%%%%%%%%%%%%%%%%%%%%%%%%%%%%%%%%%%%%%%%
\end{document}